\documentclass{article}

\PassOptionsToPackage{authoryear, round}{natbib}
\usepackage{iclr2027_conference_added_preprint_flag,times}

\usepackage[utf8]{inputenc} %
\usepackage[T1]{fontenc}    %
\usepackage{hyperref} %
\usepackage{url}            %
\usepackage{booktabs}       %
\usepackage{amsfonts}       %
\usepackage{nicefrac}       %
\usepackage{microtype}      %
\usepackage{xcolor}         %
\usepackage{algorithm}
\usepackage{algorithmic}
\usepackage{wrapfig}
\usepackage{yao}

\hypersetup{hidelinks}

\title{Towards Optimality in Active Preference Learning}

\author{%
  Yao Zhao \\
  Department of Computer Science\\
  University of Arizona\\
  \texttt{yaoz@arizona.edu} \\
  \And
  Kwang-Sung Jun \\
  Graduate School of Artificial Intelligence\\
  Department of Computer Science and Engineering\\
  Pohang University of Science and Technology\\
  \texttt{kwangsungjun@postech.ac.kr}
}

\iclrpreprint
\begin{document}

\maketitle

\begin{abstract}
Aligning large language models (LLMs) depends on high-quality datasets of human preference labels, which are costly to collect. Although active learning has been studied to improve sample efficiency relative to passive collection, many existing approaches adopt classical experimental design criteria such as G- or D-optimality. These objectives are not tailored to the structure of preference learning, leaving open the design of problem-specific algorithms. In this work, we identify a simple intuition specific to preference learning that calls into question the suitability of these existing design objectives. Motivated by this insight, we propose two active learning algorithms. The first provides the first instance-dependent label complexity guarantee for this setting, and the second is a simple, practical greedy method. We show the optimality of our instance-dependent term up to a constant depending on the logit bound, compared with a lower-bound characteristic. We evaluate our algorithms on real-world preference datasets and synthetic data observing improved sample efficiency compared to existing methods.
\end{abstract}

\begingroup
\setlength{\parskip}{3pt}
\setlength{\abovedisplayskip}{5pt plus 2pt minus 2pt}
\setlength{\belowdisplayskip}{5pt plus 2pt minus 2pt}
\setlength{\belowdisplayshortskip}{3pt plus 2pt minus 2pt}
\setlength{\textfloatsep}{12pt plus 2pt minus 2pt}
\setlength{\floatsep}{10pt plus 2pt minus 2pt}

\section{Introduction}
\label{sec:intro}

Reinforcement Learning from Human Feedback (RLHF) is a central post-training approach for aligning large language models (LLMs) with human preferences, addressing limitations of standard supervised fine-tuning. A core component of RLHF is learning a reward model that captures human judgments about the quality of model-generated responses. This reward model then guides optimization of the LLM, often via reinforcement learning algorithms such as Proximal Policy Optimization (PPO) \citep{ouyang2022training}.

The standard paradigm for learning a reward model is pairwise preference learning. In this setup, human labelers compare two LLM-generated responses to the same prompt and indicate which one they prefer. The resulting data are used to fit a reward model, commonly the Bradley-Terry-Luce (BTL) model \citep{bradley1952rank}. This paradigm is motivated by the cognitive observation that relative comparisons are typically easier and more reliable for humans than assigning absolute scores to individual items.

However, collecting human preference labels is expensive and time-consuming. This challenge is amplified when high-quality labels are required, such as in safety-critical applications like medical assistance, where each label may require expert knowledge and careful judgment. These constraints make data collection difficult to scale and can also compromise label quality \citep{zhang2023huatuogpt}. Consequently, it is important to develop data-efficient methods that minimize the number of queries to human labelers while maintaining reward model quality.

To this end, we cast preference learning as an active learning problem, in which an algorithm adaptively selects informative pairs to query, since \textit{not all pairs of responses are equally important}. Recent work applies active learning to this setting using well studied criteria from G-optimal and D-optimal experimental design \citep{das2025active, mehta2023sample, scheid2024optimal, mukherjee2024optimal}. However, these approaches primarily provide worst-case guarantees and do not adapt to the instance-specific difficulty of the underlying preference learning problem. As a result, they may fail to exploit the structure present in a given problem instance, leading to suboptimal sampling strategies.

In this work, we argue for an instance-dependent approach. Our method is guided by a simple intuition: pairs of responses with near-ties in preference (i.e., with small reward differences) are the most ambiguous; therefore, querying them is informative. Building on this idea, we develop two complementary algorithms. The first, based on a novel experimental design formulation, has an instance-dependent label-complexity bound whose leading term is near-optimal relative to a lower bound for logistic instances with bounded true logits, up to a constant depending on the logit bound. The second is a practical, easily implemented greedy method based on a new uncertainty sampling heuristic.

We summarize our contributions as follows:
\begin{itemize}
\item We propose a novel experimental design objective for active preference learning that, to the best of our knowledge, is the first to yield an instance-dependent label complexity guarantee for this active preference learning problem. We also show an information-theoretic lower bound and compare it with the leading term for bounded-logit logistic instances.
\item We design a simple and practical greedy algorithm based on a new uncertainty sampling heuristic. The method is easy to implement and avoids solving any complex optimization problems.
\item We validate our methods on real-world preference datasets and synthetic data. The results show that the proposed algorithms consistently improve upon existing baselines for serveral different type of experiments including reward model accuracy and best-of-N-sampling.
\end{itemize}

\textbf{Notation:} Throughout the paper, we use the following notation: $[n] = \{1, 2, \ldots, n\}$; $\triangle(\cZ)=\{\lambda\in\mathbb{R}^{\abr{\cZ}}\mid \lambda_{z\in\cZ}\ge0,\sum_{z\in\cZ}\lambda_z=1\}$ is the probability simplex over the set $\cZ$; $\vcn{z}_A=\sqrt{z^\top A z}$ is the norm of $z$ with respect to a positive definite matrix $A$; $H(\lambda,\theta)=\sum_{z\in\cZ}\lambda_z \dot{\mu}\rbr{z^\top \theta}z z^\top$, where $\dot{\mu}$ is the derivative of a function $\mu$.

\section{Related Work}

The RLHF pipeline was popularized by InstructGPT \citep{ouyang2022training}, although the underlying idea dates back to earlier work of \citet{christiano2017deep}. The now-standard optimization algorithm in this setting, Proximal Policy Optimization (PPO), was introduced in \citet{schulman2017proximal}. Subsequently, alternative approaches such as Direct Preference Optimization (DPO) have been proposed \cite{rafailov2023direct}. Despite differences in methodology, these alignment techniques all require high-quality preference datasets, motivating the development of sample-efficient data collection methods.

A recent line of work studies active learning for preference collection to improve sample efficiency \citep{das2025active, mehta2023sample, scheid2024optimal, zhangself, pmlr-v235-dwaracherla24a, muldrew2024active, mukherjee2024optimal,liu2024dual,schlaginhaufen2025efficient, kveton2025active}. \citet{mehta2023sample} was among the first to formalize active learning for RLHF, proposing an active contextual dueling bandit approach with a worst-case suboptimality gap guarantee. \citet{das2025active} subsequently relaxed the assumptions to better match the BTL model, and \citet{liu2024dual} further extended the framework to account for labelers with heterogeneous expertise. Other work adopts alternative problem formulations, including best-arm identification \citep{scheid2024optimal} and online, rather than offline, alignment \citep{zhangself}. Notably, all these results are worst-case and do not capture instance-dependent difficulty. More recently, active learning has also been developed specifically for the DPO pipeline \citep{muldrew2024active, kveton2025active}.

Many of these methods are grounded in experimental design, with D-optimal design being particularly popular \citep{mukherjee2024optimal,liu2024dual,schlaginhaufen2025efficient}. The foundational work in this method dates back to \citet{kiefer1960equivalence}. Recent papers adapt D-optimal design to preference collection in several ways. \citet{mukherjee2024optimal} proposed D-optimal experimental design algorithms for various feedback models and evaluated them using ranking based loss functions rather than the more common classification accuracy; however, their procedures are non-adaptive and use fixed allocations. \citet{liu2024dual} proposed a dual active learning approach based on D-optimal design that also models labeler expertise. Other work combines randomized exploration with D-optimal design \citep{schlaginhaufen2025efficient}. D-optimal design has also been adapted to DPO, with worst-case theoretical guarantees on logit estimation error \citep{kveton2025active}.

Beyond alignment-oriented preference learning, active learning and, more broadly, interactive machine learning are increasingly applied to other aspects of LLM use, including in-context learning \citep{margatina2023active} and prompt engineering \citep{diao2024active, shi2024efficient}.

\section{Preliminaries}
\label{sec:preliminaries}

Consider a set of prompts $\cX$ with $\abr{\cX}=n$, where each prompt $x\in\cX$ is associated with two candidate responses, $a_1$ and $a_2$. The unlabeled dataset consists of triplets $(x, a_1, a_2)$. Such response pairs are typically generated either by sampling the same LLM under different settings or by using two different LLMs. Following bandit terminology, we refer to each triplet as an arm, and let $z=\phi(x, a_1, a_2)\in \cZ \subset \mathbb{R}^d$ be a feature mapping satisfying the antisymmetry property, i.e., $\phi(x, a_1, a_2)=-\phi(x, a_2, a_1)$ for any $x, a_1, a_2$. As in \citet{das2025active}, we assume $\vcn{z} \le 1$ for all $z\in\cZ$. The preference label is $y \in \cbr{1,0}$, where $y=1$ indicates that the first response is preferred and $y=0$ indicates that the second is preferred. We write $a_1 \succeq a_2$ to denote that $a_1$ is preferred over $a_2$.

We consider a realizable setting. There exists a true parameter $\theta^* \in \mathbb{R}^d$ such that the preference model satisfies, for any $z\in \cZ$,
\begin{align}
p(y=1|x, a_1, a_2) = \psi\rbr{z^\top \theta^*}, \label{eq:pref_model}
\end{align}
where $\psi$ is a monotone link function mapping the score $z^\top \theta^*$ to a probability. This framework includes several standard choices, including the logistic link $\psi(x) = \frac{1}{1+e^{-x}}$, the linear link $\psi(x)=\frac{x+1}{2}$ for $x \in [-1, 1]$, and the Gaussian CDF link $\psi(x) = \Phi(x)$, where $\Phi$ is the cumulative distribution function of the standard Gaussian distribution. In the preference learning literature, the logistic model is often referred to as the BTL model, and it is the primary focus of this paper. We assume that the minimum absolute margin is strictly positive, $\Delta=\min_{z\in\cZ}|z^\top\theta^*|>0$.

We formulate the pairwise preference learning task as an active classification problem for the sign of $z^\top\theta^*$. The goal is to learn a classifier that correctly predicts the preferred response for every pair by actively and adaptively selecting an arm at each round and querying human annotators for the corresponding preference label.

At every round $t$, based on all the data collected up to time $t$, the learner selects an arm and queries the human annotator for its preference label. The learner then receives the label that is generated by \cref{eq:pref_model}. This sequential protocol can be seen as a special case of a batched setting, which is more realistic in large-scale label collection where the learner can select a batch of arms and query for their preference labels at once \citep{scheid2024optimal}. The labeled preference dataset at time $t$ can thus be represented as $\cD_t = \cbr{(z_{s}, y_{s})\mid s\in[t]}$. A classifier $\hat{\theta}_t$ can be trained based on $\cD_t$ by solving the following maximum likelihood estimation (MLE) problem,
\begin{align}
  \hat{\theta}_t = \arg\max_{\theta \in \mathbb{R}^d} \sum_{s=1}^t y_s \log \psi(z_s^\top \theta) + (1 - y_s) \log(1 - \psi(z_s^\top \theta)). \label{eq:ols_theta}
\end{align}
An arm $z$ is said to be classified correctly by the classifier $\hat{\theta}_t$ if $\sign(z^\top \hat{\theta}_t) = \sign(z^\top \theta^*)$. The goal of the learner is to design a strategy that can classify all the arms correctly with high probability. We define the following notion of probably approximately correct (PAC) strategy.
\begin{definition}{($\delta$-PAC)}
  A strategy is said to be $\delta$-PAC if, for any problem instance $(\cZ, \theta^*)$ and any $\delta > 0$, it can be executed in a finite number of steps and with probability at least $1-\delta$, it outputs a classifier that can classify all the arms correctly.
\end{definition}
Our goal is to design a $\delta$-PAC strategy, while minimizing the number of queries made to human annotators.

\begin{remark}
   Previous works \citet{das2025active} and \citet{mehta2023sample} use the notion of a suboptimality gap to evaluate the performance of the learned reward model, which is defined as the maximum difference between the true reward of the best response and the true reward of the response selected by the policy, often the \emph{greedy} policy defined by $\hat{\theta}$. Formally, it can be expressed as
  \begin{align*}
    \operatorname{SubOpt}(\pi)=\max_{x\in\cX} \max_{a\in\cA} \rbr{r^*(x, a) - r^*(x, \pi(x))},
  \end{align*}
  where $\pi(x)$ is the policy defined by $\hat{\theta}$, $\cA$ is a set of all responses, and $r^*(x, a)$ is a reward function. This measure has two key limitations. First, it focuses exclusively on the worst-case context, indeed all of the previous works with this measure so far are only able to show a worst-case guarantee in the end, and fail to capture instance-dependent behavior, ignoring the structural nuances of the problem instance. In contrast, we will show later in this paper that our approach is able to provide an instance-dependent guarantee. 
  Second, it is insensitive to finer-grained preferences between responses, which is crucial for the alignment task. For example, consider a prompt $x$ with three different responses satisfying $a_1 \succeq a_2 \succeq a_3$, a reward model that correctly identifies $a_1$ as the best response (resulting in zero $\operatorname{SubOpt}$) may still fail to distinguish the preference between $a_2$ and $a_3$. Indeed, a reward model is not just to pick a winner, but to provide a nuanced understanding of the relative quality of all responses. Note that given any number of responses, we can always form pairs from them.
\end{remark}

\subsection{Confidence Interval}
Our approach relies on two key ingredients: a confidence width of the MLE estimator in \cref{eq:ols_theta} and the relative position of the resulting confidence interval, as we clarify in the next section. For completeness, we state here a well-grounded confidence interval in the literature, used throughout the rest of the paper. However, our framework is general and does not depend on this particular choice.

\begin{theorem}
\label{thm:conf_logi}
(Theorem~1 of \citet{jun2021improved})
Let $\delta \le e^{-1}$. Let $\hat{\theta}_t$ be the solution of \cref{eq:ols_theta} where, for every $s \in [t]$, $y_s$ is conditionally independent from $\{z_i\}_{i=1}^t \setminus \{z_s\}$ given $z_s$ (i.e., the $z_s$'s are a fixed design).
Fix $z \in \mathbb{R}^d$ with $\|z\| \le 1$. Let $t_{\text{eff}}$ be the number of distinct vectors in $\{z_s\}_{s=1}^t$.
Define
\begin{align*}
\gamma(d)
&= 64 \Bigl(d \log(6) + \log\Bigl(\frac{2 + t_{\text{eff}}}{\delta}\Bigr)\Bigr).
\end{align*}
Define the event
\begin{align*}
\mathcal{E}_{\text{var}}=\Bigl\{\forall z',\;\frac{1}{\sqrt{2.2}}\|z'\|_{H'(\cbr{z_s,s\in[t]},\theta^*)^{-1}}\le\|z'\|_{H'(\cbr{z_s,s\in[t]},\theta_t)^{-1}}\le\sqrt{2.2}\,\|z'\|_{H'(\cbr{z_s,s\in[t]},\theta^*)^{-1}}\Bigr\}.
\end{align*}
If
\begin{align*}
\xi_t^2
:= \max_{s \in [t]}
\|z_s\|_{H'(\cbr{z_s,s\in[t]},\theta^*)^{-1}}^2
\le \frac{1}{\gamma(d)},
\end{align*}
then
\begin{align*}
\mathbb{P}\Bigl(\bigl| z^\top (\hat{\theta}_t - \theta^*) \bigr|> D_{\delta,t}(z):=2.4 \|z\|_{H'(\cbr{z_s,s\in[t]},\theta^*)^{-1}}\cdot\sqrt{ \log \frac{2(2 + t_{\text{eff}})}{\delta} },\ \mathcal{E}_{\text{var}}\Bigr)\le \delta,
\end{align*}
where $H'(\cD,\theta) = \sum_{z\in\cD} \dot{\mu}\rbr{z^\top \theta} z z^\top$ is the Fisher information matrix at $\theta$.
\end{theorem}

\subsection{Intuition}
\label{sec:intuition}

\begin{wrapfigure}{r}{0.40\textwidth}
\centering
\resizebox{\linewidth}{!}{%
\begin{tikzpicture}
\begin{axis}[
    axis x line=middle,
    axis y line=middle,
    axis line style={->},
    xmin=0.5, xmax=10,
    ymin=-0.5,  ymax=1,
    xtick=\empty, ytick=\empty,
    extra y ticks={0},            
    extra y tick labels={0},
    clip=false,
    enlarge x limits={upper=0.2}, 
    xlabel={},          
    ylabel={$z^\top \theta^*$},    
]

\pgfplotsset{
    myerr/.style={
        error bars/.cd,
        y dir=both,
        y explicit,
        error mark=*,                  
        error bar style={solid, thick}
    },
    mydot/.style={
        only marks,
        mark=*,
        mark options={fill=black, draw=black},
        black
    }
}

\addplot+[mydot, myerr, error bar style={green, thick}]
    coordinates {(2,0.65) +- (0,0.3)};
\node[green, anchor=west, align=left, font=\small] at (axis cs:2.2, 0.65) {\textbf{Arm 1} \\ 0.8};

\addplot+[mydot, myerr, error bar style={red, thick}]
    coordinates {(4,-0.1) +- (0,0.2)};
\node[red, anchor=west, align=left, font=\small] at (axis cs:4.2, -0.1) {\textbf{Arm 2} \\ 0.4};

\addplot+[mydot, myerr, error bar style={green, thick}]
    coordinates {(6,-0.25) +- (0,0.15)};
\node[green, anchor=west, align=left, font=\small] at (axis cs:6.2, -0.25) {\textbf{Arm 3} \\ 0.3};

\addplot+[mydot, myerr, error bar style={red, thick}]
    coordinates {(8,0.2) +- (0,0.3)};
\node[red, anchor=west, align=left, font=\small] at (axis cs:8.2, 0.2) {\textbf{Arm 4} \\ 0.8};

\end{axis}
\end{tikzpicture}
}
\caption{Intuition for our proposed method. An arm's sign is considered confidently determined (green) if its confidence interval does not overlap with zero, and uncertain (red) otherwise. The numbers indicate the widths of the confidence intervals.}
\label{fig:rlhf_intuition}
\vspace{-1.4em}
\end{wrapfigure}

The core intuition of our approach is to leverage, for each arm $z \in \mathcal{Z}$, the \textit{individual} location of the confidence interval for $z^\top\theta^*$. 
This contrasts with most existing works, which primarily aim to reduce \textit{overall} uncertainty in the estimate of $\theta^*$.
For instance, D-optimal experimental design minimizes the volume of the confidence ellipsoid for $\theta^*$, while G-optimal design minimizes the maximum prediction variance over arms. Equivalently, these criteria depend only on the widths of the confidence intervals for $z^\top\theta^*$ and ignore their locations.
\looseness=-1

We illustrate our intuition in \cref{fig:rlhf_intuition} using four arms with distinct confidence intervals on $z^\top\theta^*$. The varying widths of these intervals reflect different levels of uncertainty at a given time $t$ during the learning process. Algorithms that are designed to minimize uncertainty of each individual arm, such as G-optimal experimental design \citep{das2025active}, would prioritize arms with larger uncertainty, i.e., wide confidence intervals, such as arms 1 and 4 in the figure. Although this strategy may seem reasonable from a general uncertainty reduction perspective and is often used in many heuristics, from a sign classification perspective, this strategy can be suboptimal. 
For example, for arm 1, the entire confidence interval is above zero, so we can predict its sign correctly with high probability despite the interval’s large width. 
In contrast, the confidence interval for arm 2 includes zero. 
This indicates that under the current estimate $\hat{\theta}$, its sign remains ambiguous, even though its associated uncertainty width is relatively small. 
Therefore, intuitively it is more beneficial to pull arm 2 to reduce its uncertainty until its confidence interval is fully separated from zero.

In the next two sections, we formalize the above intuition into two algorithm designs.

\section{Active Preference Learning via Experimental Design}
\label{sec:exp_design_rlhf}
In this section, we develop a $\delta$-PAC algorithm for the active preference learning problem based on the experimental design framework.

\subsection{Optimal Allocation}
To motivate our algorithm, we first analyze the optimal query allocation in an oracle setting where the true parameter $\theta^*$ is known. This allocation represents the ideal distribution of preference queries across pairs of arms. Subsequently, we will introduce an adaptive algorithm that approximates this optimal allocation without knowledge of $\theta^*$.

We begin by defining $\mathcal{C}_\theta$ as the set of all parameters $\theta$ that correctly classify all arms,
\begin{align*}
  \mathcal{C}_\theta: = \cbr{\theta\in\mathbb{R}^d\mid \forall z\in\mathcal{Z}, \sign\rbr{z^\top\theta}=\sign\rbr{z^\top\theta^*}}.
\end{align*}
Intuitively, the algorithm should be designed to return a classifier $\hat{\theta}$ that belongs to the set $\mathcal{C}_\theta$ with high probability and minimal label complexity. This requires an optimal allocation over the arms $\mathcal{Z}$ that efficiently exploits the structure of $\mathcal{C}_\theta$, ensuring the classifier $\hat{\theta}$ converges to this set as quickly as possible. We define
\begin{align*}
  \mathcal{Z}^+_{\theta^*}: = \cbr{z\in\mathcal{Z}\mid z^\top\theta^*>0} \quad \text{and} \quad \mathcal{Z}^-_{\theta^*}: = \cbr{z\in\mathcal{Z}\mid z^\top\theta^*<0}.
\end{align*}
By \cref{thm:conf_logi}, we have
\begin{align}
  z^\top\theta^*-z^\top\hat{\theta}_t\leq \cC\cdot\vcn{z}_{{H'(\cbr{z_s,s\in[t]},\theta^*)^{-1}}}\sqrt {\log\rbr{t/\delta}}, \label{ieq:conf_bound}
\end{align}
where $\cC$ is a positive constant introduced to simplify notation. For any arm $z\in\mathcal{Z}^+_{\theta^*}$, the requirement $z^\top\hat{\theta}_t>0$ is equivalent to:
\begin{align}
  z^\top\theta^*-z^\top\hat{\theta}_t<z^\top\theta^*, \label{ieq:equivalence}
\end{align}
Given the confidence bound in \cref{ieq:conf_bound}, a sufficient condition for \cref{ieq:equivalence} to hold is
\begin{align}
  \cC\cdot\vcn{z}_{{H'(\cbr{z_s,s\in[t]},\theta^*)^{-1}}}\sqrt {\log\rbr{t/\delta}}<z^\top\theta^*. \label{ieq:sufficient_condition}
\end{align}
To ensure $z^\top\hat{\theta}_t>0$ for all $z\in\mathcal{Z}^+_{\theta^*}$, this must hold for the worst case, leading to the sufficient condition: 
\begin{align*}
\frac{1}{\cC\cdot\sqrt{\log\rbr{t/\delta}}}\ge
\max_{z\in\cZ}\frac{\vcn{z}_{{H'(\cbr{z_s,s\in[t]},\theta^*)^{-1}}}}{z^\top\theta^*}=\max_{z\in\cZ}\frac{\vcn{z}_{{\rbr{t\sum_{z\in\cZ}\lambda_{z,t}\,\dot{\mu}\rbr{z^\top \theta^*}\,z z^\top}}^{-1}}}{z^\top\theta^*},
\end{align*}
where $\lambda_{z,t} = \sum_{s=1}^t\ii{z_s = z}/t$ is the empirical frequency of arm $z$ till time $t$. This motivates choosing a design that satisfies this sufficient condition as quickly as possible,
\begin{align}
  \label{eq:design_posi}
  \lambda^*_+=\arg\min_{\lambda\in\Delta(\cZ)}\max_{z\in\cZ}\fr{\vcn{z}_{{\rbr{\sum_{z\in\cZ}\lambda_z \dot{\mu}\rbr{z^\top \theta^*}z z^\top}}^{-1}}}{z^\top\theta^*}.
\end{align}
Similarly, for an arm $z\in\mathcal{Z}^-_{\theta^*}$, a parallel argument leads to the design:
\begin{align}
  \label{eq:design_nega}
  \lambda^*_-=\arg\min_{\lambda\in\Delta(\cZ)}\max_{z\in\cZ}\frac{\vcn{z}_{{\rbr{\sum_{z\in\cZ}\lambda_z \dot{\mu}\rbr{z^\top \theta^*}z z^\top}}^{-1}}}{-z^\top\theta^*}.
\end{align}
Combining the objectives from \cref{eq:design_posi} and \cref{eq:design_nega} yields a unified design that ensures the parameter estimate converges to the true parameter $\theta^*$ at a fast rate,
\begin{align*}  \lambda^*=\arg\min_{\lambda\in\Delta(\cZ)}\max_{z\in\cZ}\fr{\vcn{z}_{{\rbr{\sum_{z\in\cZ}\lambda_z \dot{\mu}\rbr{z^\top \theta^*}z z^\top}}^{-1}}}{\abr{z^\top\theta^*}}.
\end{align*}

\paragraph{Comparison to G-Optimal and XY-Optimal Designs.}
Readers familiar with experimental design methods, particularly in multi-armed bandits, may have noticed a similarity between our design objective and the G-optimal or XY-optimal design objective as frequently used in best arm identification problems \citep{soare2014best, NEURIPS2024_zhao, fiez2019sequential, yang2022minimax}. However, there are crucial differences. The G-optimal design aims to minimize the maximum prediction variance,
\begin{align*}
    \arg\min_{\lambda\in\Delta(\cZ)}\max_{z\in\cZ}\vcn{z}_{{\rbr{\sum_{z\in\cZ}\lambda_z z z^\top}}^{-1}}.
\end{align*}
The core distinction in our objective is the instance-dependent denominator $|z^\top\theta^*|$, which is absent in the G-optimal design objective. As explained in \cref{sec:intuition}, this term augments the G-optimal design objective by considering the confidence interval's location relative to the decision boundary. This augmentation is the key to achieving our instance-dependent results in the next section, moving beyond worst-case guarantees \citep{das2025active,mukherjee2024optimal,liu2024dual}.

Our objective also differs from the XY-optimal design \citep{fiez2019sequential, NEURIPS2024_zhao}, which is defined as,
\begin{align*}
    \arg\min_{\lambda\in\Delta(\cZ)}\max_{z_1,z_2\in\cZ}\frac{\vcn{z_1-z_2}_{{\rbr{\sum_{z\in\cZ}\lambda_z z z^\top}}^{-1}}}{|(z_1-z_2)^\top\theta^*|}.
\end{align*}
While the XY-optimal design also has an instance-dependent denominator, it is tailored specifically for best arm identification, as the allocation focuses on reducing uncertainty in the direction of $(z_1-z_2)$ for all $z_1, z_2 \in \cZ$. This is suitable for identifying a single winner among all arms. In contrast, our objective is designed to classify all arms, focusing on resolving the uncertainty of each individual arm $z$ with respect to the decision boundary. There are also other differences between our problem and the best arm identification problem, which we explain in the remark of next section.

\subsection{Algorithm Design with Unknown \texorpdfstring{$\theta^*$}{theta*}}
Guided by the oracle allocation from the previous section, we now present an adaptive algorithm that approximates this allocation without access to the true parameter $\theta^*$. Our approach is inspired by the sequential experimental design framework \cite{fiez2019sequential, jun2021improved, NEURIPS2024_zhao}, which has been applied successfully to various multi-armed bandit problems with linear structure.

The key idea is that, to classify all arms with high confidence, it suffices to shrink the confidence interval for each $z^\top\theta^*$ until it no longer overlaps zero. Once an interval is separated from zero, the sign of the corresponding $z^\top\theta^*$ is determined with high confidence by \cref{ieq:equivalence}-\cref{ieq:sufficient_condition}. The number of samples required for a given arm $z\in\cZ$ depends on the magnitude of $z^\top\theta^*$. The oracle design from the previous section provides an allocation that achieves this separation efficiently by adapting to the structure of the problem instance.
\looseness=-1

The full algorithm is presented in \cref{alg:exp_design_logi}. It proceeds in two phases: a warm-up procedure to obtain an initial estimate of $\theta^*$, and a main sequential experimental design procedure. The warm-up phase, detailed in \cref{alg:warmup}, collects an initial dataset by sampling according to a G-optimal design.
\looseness=-1

The main phase of the algorithm operates in rounds. At each round $\ell$, we maintain a set of 'active' arms, $\mc{Z}_\ell$, which includes all arms whose signs have not yet been confidently determined by the previous estimate, $\hat{\theta}_{\ell-1}$. We start with $\mc{Z}_1=\mc{Z}$. The algorithm also uses an exponentially decreasing precision parameter, $\eps_\ell$, that controls the effective distance to the decision boundary for the active set. In each round, we compute a new design, $\lambda_\ell$, based on the estimate $\hat{\theta}_{\ell-1}$ and the active set $\mc{Z}_\ell$. Arms that are far from the decision boundary, i.e., where $\abr{z^\top\theta^*}$ is large, are expected to be classified with high confidence in early rounds and are thus removed from the active set.

We note that the warm-up procedure requires knowledge of a lower bound on the minimum derivative of the link function at the true parameter.

\begin{algorithm}[t]
  \caption{Warm-up procedure}
  \label{alg:warmup}
  \begin{algorithmic}[1]
    \STATE \textbf{Input:} $\mc{Z}, \delta, \omega, \kappa_0=\min_{z\in\mc{Z}}\dot{\mu}\rbr{z^\top\theta^*}$
    \STATE \textbf{Initialize:}
    \[
      \lambda_0=\arg\min_{\lambda\in\Delta(\mc{Z})}\;
      \max_{z\in\mc{Z}}\|z\|^2_{(\sum_{z\in \mathcal{Z}}\lambda_z zz^{\top})^{-1}}
    \]
    \STATE \textbf{Set:}
    $N_0=\ur{3(1+\omega)\kappa_0^{-1}d\gamma(d)\log\rbr{2\abr{\mc{Z}}\rbr{2+\abr{\mc{Z}}}/\delta}}$
    \STATE Pull arms in $Z_{N_0}=\text{ROUND}(\lambda_0,N_0)$ and observe $Y_{N_0}$
    \STATE \textbf{return:} MLE $\hat{\theta}_0$ by \cref{eq:ols_theta}
  \end{algorithmic}
\end{algorithm}

\begin{algorithm}[t]
  \caption{Experimental design algorithm}
  \label{alg:exp_design_logi}
  \begin{algorithmic}[1]
  \STATE \textbf{Input} $\mc{Z}, \delta, \omega$
  \STATE \textbf{Initialize:} $\ell=1, \mc{Z}_1=\mc{Z}, \eps_\ell=2^{-\ell+1}, Q=\emptyset$
  \STATE $\hat{\theta}_0= $ \cref{alg:warmup} (warm-up procedure)
  \STATE \textbf{Set}\\
  \quad$f_1(\hat{\theta}_{\ell-1},\lambda):=\gamma(d)\max_{z\in\mc{Z}}\|z\|^2_{H(\lambda, \hat{\theta}_{\ell-1})^{-1}}$\\
  \quad$f_2(\mc{Z}_\ell,\hat{\theta}_{\ell-1},\lambda):=2.4^2\eps_\ell^{-2}\max_{z\in\mc{Z}_\ell}\|z\|^2_{H(\lambda, \hat{\theta}_{\ell-1})^{-1}}$\\
  \quad$f(\mc{Z}_\ell,\hat{\theta}_{\ell-1},\lambda):=f_1(\hat{\theta}_{\ell-1},\lambda)\vee f_2(\mc{Z}_\ell,\hat{\theta}_{\ell-1},\lambda)$

  \STATE \textbf{while} $\abr{\mc{Z}_\ell}\ge1$ \textbf{do}
  \STATE \hspace{1em} $\lambda_\ell=\arg\min_{\lambda\in\Delta(\mc{Z})}f(\mc{Z}_\ell,\hat{\theta}_{\ell-1},\lambda)$
  \STATE \hspace{1em} $\rho(\mc{Z}_\ell)=\min_{\lambda\in\Delta(\mc{Z})}f(\mc{Z}_\ell,\hat{\theta}_{\ell-1},\lambda)$
  \STATE \hspace{1em} $N_\ell=\ur{3(1+\omega)\rho(\mc{Z}_\ell)\log\rbr{\fr{2\ell^2\abr{\mc{Z}}\rbr{2+\abr{\mc{Z}}}}{\delta}}}\vee r(\omega)$
  \STATE \hspace{1em} $Z_{\ell}=\text{ROUND}(\lambda_\ell,N_\ell)$
  \STATE \hspace{1em} $Q\gets Q\cup Z_{\ell}$
  \STATE \hspace{1em} Pull arms in $Z_{\ell}$ and observe $Y_{\ell}$
  \STATE \hspace{1em} Compute MLE $\hat{\theta}_\ell$ with data $\cbr{Z_{\ell}, Y_\ell}$ by \cref{eq:ols_theta}
  \STATE \hspace{1em} $\mc{Z}_{\ell+1}=\mc{Z}_\ell\backslash\cbr{z\in\mc{Z}_\ell\middle| \abr{z^\top\hat{\theta}_\ell}>\eps_\ell}$
  \STATE \hspace{1em} $\ell\gets \ell+1$
  \STATE \textbf{end while}

  \STATE Pull arms in $Q$ and observe $Y$
  \STATE Compute MLE $\hat{\theta}$ with data $\cbr{Q, Y}$ by \cref{eq:ols_theta}
  \STATE \textbf{Output:} $\hat{\theta}$
  \end{algorithmic}
\end{algorithm}

\begin{remark}
The final, non-adaptive sampling phase in \cref{alg:exp_design_logi} (lines 16-17) is a critical step. It is necessitated by a technical requirement of the confidence bound in \cref{thm:conf_logi} and same for many other similar confidence bounds \citep{abbasi2011improved}, which requires that data is collected non-adaptively. That is, the selection of arms cannot depend on previous observations. Consequently, the confidence guarantee would not hold if we were to compute the final MLE using only the data gathered adaptively during the main rounds. This requirement distinguishes our setting from many standard sequential design problems \citep{fiez2019sequential, jun2021improved, NEURIPS2024_zhao}, where the goal is to identify a single best arm. In contrast, our goal is to return a reward model, which necessitates a valid confidence bound on the final parameter estimate itself.
\end{remark}

We show the following two instance dependent quantities that essentially capture the complexity of the problem instance,
\begin{align*}
  \rho^{\ast}=\min_{\lambda \in \Delta(\mc{Z})}\max_{z\in \mathcal{Z}} \frac{\vcn{z}^2_{H(\lambda,\theta^*)^{-1}}}{\abr{z^\top\theta^*}^2} \quad \text{and} \quad \rho_0=\min_{\lambda\in\Delta(\mc{Z})}\max_{z\in\mc{Z}}3\gamma(d)\|z\|^2_{H(\lambda, \theta^*)^{-1}}.
\end{align*}
We define the following shorthand for simplicity before we are ready to present an upper bound for the label complexity of \cref{alg:exp_design_logi}, $\ell^*=\ur{\log\rbr{4\Delta^{-1}}}$, $\Delta=\min_{ z \in\mc{Z}}\abr{z^\top\theta^*}$, $\overline{\log}=\log\rbr{\fr{2\ell^{*2}\abr{\mc{Z}}\rbr{2+\abr{\mc{Z}}}}{\delta}}$.

\begin{theorem}
\label{thm:exp_design_complexity_logi} (Label complexity)
With a probability of at least $1-\delta$, \cref{alg:exp_design_logi} will stop after at most $c(1+\omega)\ell^*\overline{\log}\rho^*+c(1+\omega)\ell^*\overline{\log}\rho_{0}+c(1+\omega)\kappa_0^{-1}d\gamma(d)\overline{\log}+c\ell^*r(\omega)$ rounds, where $c$ is an absolute constant.
\end{theorem}
We defer the proof to \cref{proof:label_cmpx} of the Appendix.

\begin{theorem}
\label{thm:correctness} (Correctness)
With a probability of at least $1-\delta$, the reward model $\hat{\theta}$ will classify all arms in $\mathcal{Z}$ correctly, i.e.,
\begin{align*}
  \sign(z^\top\hat{\theta})=\sign(z^\top\theta^*)\quad \text{ for } \forall z\in\mc{Z}.
\end{align*}
\end{theorem}
We defer the proof to \cref{proof:correctness} of the Appendix.

Our result stands in stark contrast to existing work on active learning for RLHF. Prior works \cite{das2025active, kveton2025active, mehta2023sample, scheid2024optimal}, for instance, have analyzed various performance metrics, such as suboptimality gap, simple regret, and logit error, but have only provided worst-case guarantees. Their convergence rates, typically on the order of \(O(d/\sqrt{T})\), lack dependence on the problem instance and therefore fail to capture the intrinsic difficulty of a given problem instance.

In contrast, our results are instance-dependent. This is because our approach identifies an experimental design objective that is intrinsically tied to the problem structure. Previous works have relied on standard, off-the-shelf objectives like G-optimality and D-optimality. While these are powerful general purpose tools, they are not specifically tailored to the active learning problem in RLHF, which limits the tightness of their theoretical guarantees. We provide an information-theoretic lower bound in \cref{app-lowbound} of the Appendix and compare its instance-dependent characteristic with \(\rho^*\) below for the logistic link. We also provide a concrete instance in the next section, showing a  qualitative gap between a G-optimality method and a greedy method based on the same intuition in \cref{sec:intuition}.

\subsection{Near-optimality under bounded logits}
We now compare the leading instance-dependent term in \cref{thm:exp_design_complexity_logi} with a lower bound. Let $\cZ=\{z_1,\ldots,z_n\}$, and let $\mathcal C_{\theta^*,\cZ}$ be the set of parameters that change at least one arm's sign, as defined in \cref{thm:lower_bound_logi} of the Appendix. Define the lower-bound characteristic
\begin{align*}
  T_{\mathrm{KL}}^\star
  :=\inf_{\lambda\in\Delta(\cZ)}\;
  \sup_{\theta\in\mathcal C_{\theta^*,\cZ}}
  \frac{1}{\sum_{i=1}^n\lambda_i
  D_{\mathrm{KL}}\rbr{P_{\theta^*,i}\,\|\,P_{\theta,i}}},
\end{align*}
where $\lambda_i$ is the sampling mass on $z_i$ and $P_{\theta,i}=\mathrm{Bern}\rbr{\psi(z_i^\top\theta)}$. Then \cref{thm:lower_bound_logi} shows that $T_{\mathrm{KL}}^\star\log(1/(2.4\delta))$ lower-bounds the expected label complexity of strategies that are uniformly $\delta$-PAC over logistic instances. The lower-bound statement and its proof are given in \cref{app-lowbound} of the Appendix.

\begin{theorem}[Leading-term comparison for bounded logits]
\label{thm:bounded_logit_optimality}
Consider the logistic link $\psi(u)=1/(1+e^{-u})$. Let $\cZ=\{z_1,\ldots,z_n\}$ span $\mathbb R^d$, and suppose
\begin{align*}
  \max_{z\in\cZ}|z^\top\theta^*|\le B,
  \qquad \kappa_B:=\psi(B)(1-\psi(B)).
\end{align*}
Then
\begin{align}
  T_{\mathrm{KL}}^\star\ge8\kappa_B\rho^*.
  \label{eq:bounded_logit_comparison}
\end{align}
\end{theorem}
Thus, the leading $\rho^*$ term in \cref{thm:exp_design_complexity_logi} is upper bounded by the lower bound characteristic up to the constant $1/(8\kappa_B)$. We defer the proof to \cref{proof:bounded_logit_optimality} of the Appendix.

\section{Greedy Algorithm}
\label{sec:greedy_algorithm}

\begin{wrapfigure}{r}{0.42\textwidth}
\vspace{-3em}
\centering
\resizebox{\linewidth}{!}{%
\begin{tikzpicture}
\begin{axis}[
    axis x line=middle,
    axis y line=middle,
    axis line style={->},
    xmin=0.5, xmax=14,
    ymin=-0.5,  ymax=1,
    xtick=\empty, ytick=\empty,
    extra y ticks={0},            
    extra y tick labels={0},
    clip=false,
    enlarge x limits={upper=0.2}, 
    xlabel={},          
    ylabel={$z^\top \theta^*$},    
]

\pgfplotsset{
    myerr/.style={
        error bars/.cd,
        y dir=both,
        y explicit,
        error mark=*,                  %
        error bar style={solid, thick}
    },
    mydot/.style={
        only marks,
        mark=*,
        mark options={fill=black, draw=black},
        black
    }
}

\addplot+[mydot, myerr, error bar style={green, thick}]
    coordinates {(2,0.65) +- (0,0.3)};
\node[green, anchor=west, align=left, font=\small] at (axis cs:2.2, 0.65) {\textbf{Arm 1} \\ 0.8};

\draw[<->, dashed, thick, green] 
    (axis cs:2.2, 0) -- (axis cs:2.2, 0.35) 
    node[midway, right, font=\small, olive] {-0.35};

\addplot+[mydot, myerr, error bar style={red, thick}]
    coordinates {(5,-0.1) +- (0,0.2)};
\node[red, anchor=west, align=left, font=\small] at (axis cs:5.2, -0.1) {\textbf{Arm 2} \\ 0.4};

\draw[<->, dashed, thick, red] 
    (axis cs:5.2, 0) -- (axis cs:5.2, 0.1) 
    node[midway, right, font=\small, olive] {0.1};

\addplot+[mydot, myerr, error bar style={green, thick}]
    coordinates {(8,-0.25) +- (0,0.15)};
\node[green, anchor=west, align=left, font=\small] at (axis cs:8.2, -0.25) {\textbf{Arm 3} \\ 0.3};

\draw[<->, dashed, thick, green] 
    (axis cs:8.2, -0.1) -- (axis cs:8.2, 0) 
    node[midway, right, font=\small, olive] {-0.1};

\addplot+[mydot, myerr, error bar style={red, thick}]
    coordinates {(11,0.1) +- (0,0.3)};
\node[red, anchor=west, align=left, font=\small] at (axis cs:11.2, 0.1) {\textbf{Arm 4} \\ 0.8};

\draw[<->, dashed, thick, red] 
    (axis cs:11.2, 0) -- (axis cs:11.2, -0.2) 
    node[midway, right, font=\small, olive] {0.2};

\end{axis}
\end{tikzpicture}
}
\caption{Illustration of remaining uncertainty (in olive). The remaining uncertainty is positive only when a confidence interval contains the decision boundary (zero), indicating that the arm's sign is not yet confidently determined.}
\label{fig:greedy_intuition}
\end{wrapfigure}
In this section, we propose a simple greedy algorithm for the active preference learning problem that is directly motivated by the intuition from \cref{sec:intuition}. In contrast to the method in \cref{sec:exp_design_rlhf}, this approach avoids solving a complex optimization problem and does not require any knowledge of the minimum derivative of the link function. The key idea is to use a greedy objective that incorporates the location of each confidence interval relative to the decision boundary. As argued in \cref{sec:intuition}, relying solely on confidence width can lead to inefficient sampling. Our objective therefore uses the interval's location to augment its width.

Specifically, for each arm $z$, we define its \textit{remaining uncertainty} as the minimum amount by which its confidence interval must shrink to no longer overlap with the decision boundary (zero). This value captures the uncertainty that must be resolved to classify the arm confidently. Formally, the remaining uncertainty of arm $z$ at round $t$ is defined as:
\begin{align*}
  \mathsf{RU}_t(z)=\min\cbr{-\mathsf{LCB}_t(z), \mathsf{UCB}_t(z)},
\end{align*}
where $\mathsf{LCB}_t(z)$ and $\mathsf{UCB}_t(z)$ are the lower and upper confidence bounds of $z^\top \theta^*$ at round $t$ respectively, i.e., $\mathsf{LCB}_t(z)=z^\top\hat{\theta}_t-D_{\delta,t}(z)$ and $\mathsf{UCB}_t(z)=z^\top\hat{\theta}_t+D_{\delta,t}(z)$ with $D_{\delta,t}(z)$ being the confidence width defined in \cref{thm:conf_logi}. This concept is illustrated in \cref{fig:greedy_intuition}, which depicts four possible scenarios for the confidence interval locations.

Based on this concept, we present our algorithm in \cref{alg:greedy_obj}. The algorithm follows a simple greedy strategy: in each round, it pulls the arm with the maximum remaining uncertainty.
\looseness=-1
\begin{algorithm}[t]
  \caption{Greedy Algorithm for Active Preference Learning}
  \begin{algorithmic}[1]
  \label{alg:greedy_obj}
  \STATE \textbf{Input} $\mc{Z}$, $\delta$, labeling budget $T$
  \STATE \textbf{Initialize} $\cD \leftarrow \emptyset$
  \STATE \textbf{for} $t = 1, \dots, T$ \textbf{do}
  \STATE \quad Compute $\hat{\theta}_t$ and $(\mathsf{LCB}_t(z), \mathsf{UCB}_t(z))$ for all $z \in \mc{Z}$
  \STATE \quad  $z_t \leftarrow \argmax_{z \in \mc{Z}} \min \{ -\mathsf{LCB}_t(z), \mathsf{UCB}_t(z) \}$ with ties broken arbitrarily.
  \STATE \quad Query $z_t$ to observe preference $y_t$.
  \STATE \quad $\cD \leftarrow \cD \cup \{(z_t, y_t)\}$.
  \STATE \textbf{end for}
  \STATE \textbf{return:} MLE $\hat{\theta}$ on $\cD$ using \cref{eq:ols_theta}.
  \end{algorithmic}
\end{algorithm}

The greedy objective in \cref{alg:greedy_obj} can be interpreted as a refined measure of uncertainty for this problem. It refines the standard confidence width by also incorporating the confidence interval's location relative to the decision boundary, as shown in the following lemma.
\begin{lemma}
  \label{lem:greedy_obj_equiv}
  The greedy objective in \cref{alg:greedy_obj} can be equivalently written as
\begin{align}
  \max_z\min\cbr{-\mathsf{LCB}_t(z), \mathsf{UCB}_t(z)} = \max_z\cbr{D_{\delta,t}(z) - \abr{\hat{\theta}_t^\top z}}. \label{eq:greedy_eq}
\end{align}
\end{lemma}
We defer the proof to \cref{proof:greedy_equiv} of the Appendix. This form balances the uncertainty reduction of G-/D-optimal design with Uncertainty Sampling's preference for small estimated margins.

To demonstrate the advantage of our algorithm over a G-optimal based method, e.g., \citet{das2025active} pulling arm based on $\arg\max_{i \in [d]} \|e_i\|_{H_t'^{-1}}$, we show a simple canonical instance where their respective label complexities differ significantly. Consider the following instance. Let $\eps$ be a small enough constant, e.g., $\eps\in(0, 1/4)$. Using the stopping rule from \cref{thm:correctness_stopping_condi} in the Appendix, we have the following result.
\looseness=-1
\begin{theorem}
  \label{thm:greedy_obj_better}
  For the canonical instance $z_i = e_i$ for $i=1, \ldots, d$ and $\theta^*=(1, \ldots, 1, \eps)^\top$ for $\eps\in(0, 1/4)$, the greedy selection objective $\arg\max_{i \in [d]} \|e_i\|_{H_t'^{-1}}$ has label complexity scaling as $\tilde{O}\rbr{\fr{d}{\eps^2}}$, while with \cref{alg:greedy_obj}, the label complexity scales as $\tilde{O}\rbr{d+\fr{1}{\eps^2}}$.
\end{theorem}
We defer the proof to \cref{proof:greedy} of the Appendix.
\section{Experiments}
\vspace{-1em}
Due to space limitations, experimental details are discussed in \cref{app-experiment} of the Appendix.

\section{Conclusion}
In this work, we addressed the problem of active preference learning for RLHF by proposing a novel, tailored experimental design objective. Based on this objective, we developed an efficient algorithm that achieves an instance-dependent label complexity bound, improving upon existing worst-case only results, and showing the instance-dependent near optimality for a leading term with explicit dependence on the logit bound. We also introduced a simple and practical greedy algorithm as an alternative. However, another question remains open: how does the reward model's performance eventually impact the performance of the LLM alignment? Investigating this connection and developing a more precise experimental design objective that directly targets the overall RLHF goal are compelling directions for future work.

\endgroup

\subsection*{AI use statement}

In this work, we used generative AI tools to polish wording, correct grammar and typos, adjust LaTeX spacing, and refine figure presentation. We also used these tools for mathematical verification. We take responsibility for the final content of this work.

\bibliography{ref}
\bibliographystyle{iclr2027_conference}

\newpage

\appendix
\onecolumn

\part*{Appendix}

\section{Proofs}
\subsection{Proof of Theorem~\ref{thm:exp_design_complexity_logi}}
\label{proof:label_cmpx}
\newtheorem*{app-thm:label_cmpx}{Theorem~\ref{thm:exp_design_complexity_logi}}

\begin{app-thm:label_cmpx}(Label complexity)
With a probability of at least $1-\delta$, \cref{alg:exp_design_logi} will stop after at most $c(1+\omega)\ell^*\overline{\log}\rho^*+c(1+\omega)\ell^*\overline{\log}\rho_{0}+c(1+\omega)\kappa_0^{-1}d\gamma(d)\overline{\log}+c\ell^*r(\omega)$ rounds, where $c$ is an absolute constant.
\end{app-thm:label_cmpx}

\begin{proof}
  We prove an upper bound for the label complexity in the following.

  Define $\mathcal{S}_\ell:=\cbr{z\in\mathcal{Z}\middle| \abr{z^\top\theta^*}\le2\eps_\ell}$. Thus, by \cref{lem:exp_design_logi_ell}, with a probability at least $1-\delta$, we have $\bigcap_{\ell}\cbr{\mc{Z}_\ell \subseteq S_\ell }$ hold. This implies the following is true with a probability at least $1-\delta$ for all $\ell$
  \begin{align}
    \rho(\mc{Z}_\ell )&=\min_{\lambda\in\Delta(\mc{Z})}\max\cbr{\gamma(d)\max_{z\in\mc{Z}}\|z\|^2_{H(\lambda, \hat{\theta}_{\ell-1})^{-1}},2.4^2\eps_\ell^{-2}\max_{z\in\mc{Z}_\ell}\|z\|^2_{H(\lambda, \hat{\theta}_{\ell-1})^{-1}}}\nn\\
    &\le\min_{\lambda\in\Delta(\mc{Z})}\max\cbr{\gamma(d)\max_{z\in\mc{Z}}\|z\|^2_{H(\lambda, \hat{\theta}_{\ell-1})^{-1}},2.4^2\eps_\ell^{-2}\max_{z\in\mc{S}_\ell}\|z\|^2_{H(\lambda, \hat{\theta}_{\ell-1})^{-1}}}\tag{$\mc{Z}_\ell \subseteq S_\ell $}\\
    &\le\min_{\lambda\in\Delta(\mc{Z})}\max\cbr{3\gamma(d)\max_{z\in\mc{Z}}\|z\|^2_{H(\lambda, \theta^*)^{-1}},3\cdot2.4^2\eps_\ell^{-2}\max_{z\in\mc{S}_\ell}\|z\|^2_{H(\lambda, \theta^*)^{-1}}}\tag{\cref{lem:theta_ell_theta_star}}\\
    &\le\min_{\lambda\in\Delta(\mc{Z})}\max_{z\in\mc{Z}}3\gamma(d)\|z\|^2_{H(\lambda, \theta^*)^{-1}}+\min_{\lambda\in\Delta(\mc{Z})}\max_{z\in\mc{S}_\ell}3\cdot2.4^2\eps_\ell^{-2}\|z\|^2_{H(\lambda, \theta^*)^{-1}}\nn\\
    &=:\rho_0+\rho_{\ell}. \label{ieq:rho_ell}
  \end{align}

  By \cref{lem:exp_design_logi_ell}, with a probability at least $1-\delta$, we have $S_\ell =\emptyset$ for $\ell\ge\ell^*=\ur{\log\rbr{4\Delta^{-1}}}$. The total label complexity of \cref{alg:exp_design_logi} is the summation of the number of samples in each round, which can be upper bounded as
  \begin{align}
    &\sum_{\ell=1}^{\ell^*}N_\ell \nn\\
    =&\sum_{\ell=1}^{\ell^*}\ur{3(1+\omega)\rho(\mc{Z}_\ell)\log\rbr{\fr{2\ell^2\abr{\mc{Z}}\rbr{2+\abr{\mc{Z}}}}{\delta}}}\vee r(\omega)\nn\\
    \le&\sum_{\ell=1}^{\ell^*}4(1+\omega)\rho(\mc{Z}_\ell)\overline{\log}\vee r(\omega)\nn\\
    \le&\sum_{\ell=1}^{\ell^*}4(1+\omega)(\rho_0+\rho_{\ell})\overline{\log}\vee r(\omega) \tag{due to \eqref{ieq:rho_ell}}\\
    \le&\sum_{\ell=1}^{\ell^*}4(1+\omega)\overline{\log}\rho_{\ell}+4(1+\omega)\ell^*\overline{\log}\rho_{0}+\ell^*r(\omega).  \label{ieq:sum_Nk_logi}
  \end{align}
  On the other hand, we have
    \begin{align}
      \rho^{\ast} =& \min_{\lambda \in \Delta(\mc{Z})}\max_{z\in \mathcal{Z}} \frac{\vcn{z}^2_{H(\lambda,\theta^*)^{-1}}}{\abr{z^\top\theta^*}^2} \nn\\
      =&\min_{\lambda\in\Delta(\mc{Z})}\max_{\ell}\max_{ z \in S_\ell }\frac{\vcn{z}^2_{H(\lambda,\theta^*)^{-1}}}{\abr{z^\top\theta^*}^2} \nn\\
      \ge&\fr{1}{\ell^*}\min_{\lambda\in\Delta(\mc{Z})}\sum_{\ell=1}^{\ell^*}\max_{ z \in S_\ell }\frac{\vcn{z}^2_{H(\lambda,\theta^*)^{-1}}}{\abr{z^\top\theta^*}^2} \nn\\
      \ge&\fr{1}{16\ell^*}\sum_{\ell=1}^{\ell^*}2^{2\ell }\min_{\lambda\in\Delta(\mc{Z})}\max_{ z \in S_\ell }\vcn{z}^2_{H(\lambda,\theta^*)^{-1}}.\label{ieq:rho_star_logi}
    \end{align}
  Combining \eqref{ieq:sum_Nk_logi}, \eqref{ieq:rho_star_logi}, and $N_0$ for the warm-up, we have
  \begin{align*}
    \sum_{\ell=0}^{\ell^*}N_\ell\le c(1+\omega)\ell^*\overline{\log}\rho^*+c(1+\omega)\ell^*\overline{\log}\rho_{0}+c(1+\omega)\kappa_0^{-1}d\gamma(d)\overline{\log}+c\ell^*r(\omega),
  \end{align*}
  where $c$ is an absolute constant.
\end{proof}

\subsection{Proof of Theorem~\ref{thm:correctness}}
\label{proof:correctness}
\newtheorem*{app-thm:correctness}{Theorem~\ref{thm:correctness}}

\begin{app-thm:correctness}(Correctness)
With a probability of at least $1-\delta$, the reward model $\hat{\theta}$ will classify all arms in $\mathcal{Z}$ correctly, i.e.,
\begin{align*}
  \sign(z^\top\hat{\theta})=\sign(z^\top\theta^*)\quad \text{ for } \forall z\in\mc{Z}.
\end{align*}
\end{app-thm:correctness}

\begin{proof}
  The proof is similar to \cref{lem:exp_design_logi_ell} and \cref{lem:exp_design_correctness_logi}. Readers are recommended to read the proof of \cref{lem:exp_design_logi_ell} and \cref{lem:exp_design_correctness_logi} first. The difference is that here we need to show that the final reward model $\hat{\theta}$ will classify all arms in $\mathcal{Z}$ correctly, while in \cref{lem:exp_design_correctness_logi}, each reward model $\hat{\theta}_\ell$ is only shown to classify the arms in $\mathcal{Z}_\ell\backslash\mathcal{Z}_{\ell+1}$ correctly. Note that since $\mathcal{Z}_\ell$ will be eventually empty, for any $z\in\mathcal{Z}$, there exists $\ell$ such that $z\in \mathcal{Z}_\ell\backslash\mathcal{Z}_{\ell+1}$. So what we need to show is that for all $z\in \mathcal{Z}_\ell\backslash\mathcal{Z}_{\ell+1}$, the reward model $\hat{\theta}$ will classify $z$ correctly.
  
  With a probability of at least $1-\fr{\delta}{2\ell^2\abr{\mc{Z}}}$, we have for all $z\in\mathcal{Z}_\ell$,
  \begin{align}
    \abr{\rbr{\hat{\theta}-\theta^{*}}^\top z}\stkl{1}&2.4\vcn{z}_{H'\rbr{Q,\theta^*}^{-1}}\sqrt {\log\rbr{\fr{2\ell^2\abr{\mc{Z}}\rbr{2+\abr{\mc{Z}}}}{\delta}}}\nn\\
    \le&2.4\vcn{z}_{H'\rbr{\mc{Z}_\ell,\theta^*}^{-1}}\sqrt {\log\rbr{\fr{2\ell^2\abr{\mc{Z}}\rbr{2+\abr{\mc{Z}}}}{\delta}}}\tag{$H'\rbr{\mc{Z}_\ell,\theta^*}\preceq H'\rbr{Q,\theta^*}$}\\
    \stkl{2}&\fr{2.4\sqrt{1+\omega}\vcn{z}_{H\rbr{\lambda_\ell, \theta^*}^{-1}}}{\sqrt{N_\ell}}\sqrt {\log\rbr{\fr{2\ell^2\abr{\mc{Z}}\rbr{2+\abr{\mc{Z}}}}{\delta}}}\nn\\
    \stkl{3}&\fr{2.4\sqrt{3\rbr{1+\omega}}\vcn{z}_{H\rbr{\lambda_\ell, \hat{\theta}_{\ell-1}}^{-1}}}{\sqrt{N_\ell}}\sqrt {\log\rbr{\fr{2\ell^2\abr{\mc{Z}}\rbr{2+\abr{\mc{Z}}}}{\delta}}}\nn\\
    =&\fr{2.4\sqrt{3\rbr{1+\omega}}\vcn{z}_{H\rbr{\lambda_\ell, \hat{\theta}_{\ell-1}}^{-1}}}{2.4\sqrt{\ur{3(1+\omega)\rho(\mc{Z}_\ell)\log\rbr{\fr{2\ell^2\abr{\mc{Z}}\rbr{2+\abr{\mc{Z}}}}{\delta}}}\vee r(\omega)}} \sqrt {\log\rbr{\fr{2\ell^2\abr{\mc{Z}}\rbr{2+\abr{\mc{Z}}}}{\delta}}}\nn\\
    \le&\eps_\ell, \label{ieq:conf_width_eps_ell}
  \end{align}

where $b_1$ is due to \cref{thm:conf_logi}, $b_2$ is due to \cref{lem:rounding}, $b_3$ is due to \cref{lem:theta_ell_theta_star}. Consider an arm $z\in\mathcal{Z}_\ell\backslash\mathcal{Z}_{\ell+1}$. We have $\abr{z^\top\hat{\theta}_\ell}>\eps_\ell$ by the elimination condition in \cref{alg:exp_design_logi}. Thus by \cref{ieq:conf_width_eps_ell} and \cref{thm:correctness_stopping_condi} we know that $\hat{\theta}_\ell$ will classify $z$ correctly. By a union bound same as \cref{ieq:logi_good_event}, we know that with a probability at least $1-\delta$, all of arms in $\mathcal{Z}$ will be classified correctly by $\hat{\theta}$. 
\end{proof}

\begin{lemma}
  \label{lem:exp_design_logi_ell}
  With a probability of at least $1-\delta$, for all $z\in\mc{Z}_\ell$, we have $\abr{z^\top\theta^*}\le2\eps_\ell$, for all $\ell\ge1$.
  \end{lemma}
  \begin{proof}
  With a probability of at least $1-\fr{\delta}{2\ell^2\abr{\mc{Z}}}$, we have for all $z\in\mathcal{Z}_\ell$,
  \begin{align}
    \abr{z^\top\rbr{\hat{\theta}_\ell-\theta^{*}}}\stkl{1}&2.4\vcn{z}_{H'\rbr{\mc{Z}_\ell,\theta^*}^{-1}}\sqrt {\log\rbr{\fr{2\ell^2\abr{\mc{Z}}\rbr{2+\abr{\mc{Z}}}}{\delta}}}\nn\\
    \stkl{2}&\fr{2.4\sqrt{1+\omega}\vcn{z}_{H\rbr{\lambda_\ell, \theta^*}^{-1}}}{\sqrt{N_\ell}}\sqrt {\log\rbr{\fr{2\ell^2\abr{\mc{Z}}\rbr{2+\abr{\mc{Z}}}}{\delta}}}\nn\\
    \stkl{3}&\fr{2.4\sqrt{3\rbr{1+\omega}}\vcn{z}_{H\rbr{\lambda_\ell, \hat{\theta}_{\ell-1}}^{-1}}}{\sqrt{N_\ell}}\sqrt {\log\rbr{\fr{2\ell^2\abr{\mc{Z}}\rbr{2+\abr{\mc{Z}}}}{\delta}}}\nn\\
    =&\fr{2.4\sqrt{3\rbr{1+\omega}}\vcn{z}_{H\rbr{\lambda_\ell, \hat{\theta}_{\ell-1}}^{-1}}}{2.4\sqrt{\ur{3(1+\omega)\rho(\mc{Z}_\ell)\log\rbr{\fr{2\ell^2\abr{\mc{Z}}\rbr{2+\abr{\mc{Z}}}}{\delta}}}\vee r(\omega)}} \sqrt {\log\rbr{\fr{2\ell^2\abr{\mc{Z}}\rbr{2+\abr{\mc{Z}}}}{\delta}}}\nn\\
    \le&\eps_\ell, \label{ieq:exp_design_logi_ell}
  \end{align}
where $b_1$ is due to \cref{thm:conf_logi}, $b_2$ is due to \cref{lem:rounding}, $b_3$ is due to \cref{lem:theta_ell_theta_star}.
  
Define a good event $\mc{E}_{\ell,z}$ for each $\ell$ and $z\in\mc{Z}_\ell$ as
    \begin{align*}
      \mc{E}_{\ell,z}=\cbr{\abr{z^\top\rbr{\hat{\theta}_\ell-\theta^{*}}}\le\eps_\ell }.
    \end{align*}
    We claim that with a probability at least $1-\delta$, the event $\bigcap_{\ell=1}^{\infty}\bigcap_{z\in\mc{Z}_\ell}\mc{E}_{\ell,z}$ holds. It can be proved by a union bound as follows,
    \begin{align}
      \pp{\rbr{\bigcap_{\ell=1}^{\infty}\bigcap_{z\in\mc{Z}_\ell}\mc{E}_{\ell,z}}^c}\le&\sum_{\ell=1}^{\infty}\sum_{z\in\mc{Z}_\ell}\pp{\mc{E}_{\ell,z}^c}\nn\\
      \le&\sum_{\ell=1}^{\infty}\sum_{z\in\mc{Z}_\ell}\fr{\delta}{2\ell^2\abr{\mc{Z}}}\nn\\
      \le&\fr{\delta}{2}\sum_{\ell=1}^{\infty}\fr{1}{\ell^2}\nn\\
      \le&\delta, \label{ieq:logi_good_event}
    \end{align}
    where the last step is by the fact that $\sum_{\ell=1}^{\infty}\fr{1}{\ell^2}=\fr{\pi^2}{6}<2$. Under the event $\bigcap_{\ell=1}^{\infty}\bigcap_{z\in\mc{Z}_\ell}\mc{E}_{\ell,z}$, we show the following.
  
  For all $z\in\mathcal{Z}_\ell$ satisfying $\abr{z^\top\theta^*}>2\eps_\ell$, we have, if $z^\top\theta^*>2\eps_\ell$,
  \begin{align}
    z^\top\hat{\theta}_\ell=&z^\top\rbr{\hat{\theta}_\ell-\theta^{*}}+z^\top\theta^*\nn\\
    >&-\eps_\ell+2\eps_\ell\nn\\
    >&\eps_\ell. \label{ieq:exp_design_logi_eps_1}
  \end{align}
  If $z^\top\theta^*<-2\eps_\ell$,
  \begin{align}
    z^\top\hat{\theta}_\ell=&z^\top\rbr{\hat{\theta}_\ell-\theta^{*}}+z^\top\theta^*\nn\\
    <&\eps_\ell-2\eps_\ell\nn\\
    <&-\eps_\ell. \label{ieq:exp_design_logi_eps_2}
  \end{align}
  Combining \cref{ieq:exp_design_logi_eps_1} and \cref{ieq:exp_design_logi_eps_2}, we obtain $\abr{z^\top\hat{\theta}_\ell}>\eps_\ell$, which means $z$ satisfying $\abr{z^\top\theta^*}>2\eps_\ell$ will be removed from $\mc{Z}_\ell$ at the end of round $\ell$ according to the elimination condition in \cref{alg:exp_design_logi}.
  \end{proof}

\begin{lemma}
  \label{lem:exp_design_correctness_logi}
  For all $\ell\ge1$, with a probability at least $1-\delta$, the arms in $\mathcal{Z}_\ell\backslash\mathcal{Z}_{\ell+1}$ are all classified correctly by $\hat{\theta}_\ell$.
\end{lemma}

\begin{proof}
This result directly follows from \cref{ieq:exp_design_logi_ell}, \cref{thm:correctness_stopping_condi}, and the elimination condition in \cref{alg:exp_design_logi}. 

By \cref{ieq:exp_design_logi_ell}, we know that the confidence width of $z^\top\theta^*$ is at most $\eps_\ell$ for all $z\in\mathcal{Z}_\ell$. By the elimination condition in \cref{alg:exp_design_logi}, we have $\abr{z^\top\hat{\theta}_\ell}>\eps_\ell$ for all $z\in\mathcal{Z}_\ell\backslash\mathcal{Z}_{\ell+1}$, thus the sufficient condition in \cref{thm:correctness_stopping_condi} is satisfied, which means $\hat{\theta}_\ell$ will classify all arms in $\mathcal{Z}_\ell\backslash\mathcal{Z}_{\ell+1}$ correctly with a probability at least $1-\delta$.
\end{proof}

\begin{lemma}
  \label{lem:theta_ell_theta_star}
  (Lemma 7 of \citet{jun2021improved}) 
  For each $\ell\ge0$, define the event, 
  \begin{align*}
    \mc{R}_{\ell}=\cbr{\fr{1}{3}H(\lambda, \theta^*)\le H(\lambda, \hat{\theta}_{\ell})\le 3H(\lambda, \theta^*), \forall\lambda\in\Delta(\mc{Z})}.
  \end{align*}
  We have that 
  \begin{itemize}
    \item $\pp{\mc{R}_{0}}\ge1-2\delta$ and $\pp{\mc{R}_{\ell}\mid \mc{R}_{\ell-1},\cdots,\mc{R}_{0}}\ge1-2\delta$ for all $\ell\ge1$;
    \item $\max_{z\in\mc{Z}}\|z\|^2_{H'(\mc{Z}_\ell,\theta^*)^{-1}}\le1/\gamma(d)$.
  \end{itemize}
\end{lemma}

\begin{lemma}
  \label{lem:rounding}
  (Lemma 13 of \citet{jun2021improved})
  Assume that $\lambda \in \Delta_{\mathcal{X}}$, and that we have sampled $x_1, \dots, x_n \sim \text{ROUND}(\lambda, n, \epsilon)$ with $n \geq r(\epsilon) = \frac{(d(d+1)+2)}{\epsilon}$, and $\epsilon \leq 1$. Then, for any $\theta$,
\[
\sum_{s=1}^n \dot{\mu}(x_s^\top \theta) x_s x_s^\top \succeq \frac{n}{1+\epsilon} \sum_{x \in \mathcal{X}} \lambda_x \dot{\mu}(x^\top \theta) xx^\top.
\]
This in particular implies
\begin{itemize}
  \item \textit{For any $z$,}
  \[
  \|z\|^2_{\left(\sum_{s=1}^n \dot{\mu}(x_s^\top \theta) x_s x_s^\top\right)^{-1}} 
  \leq \frac{(1+\epsilon)}{n} \|z\|^2_{\left(\sum_{x \in \mathcal{X}} \lambda_x \dot{\mu}(x^\top \theta) xx^\top\right)^{-1}}
  \]  
  \item \[
  \lambda_{\min}\left(\sum_{s=1}^n \dot{\mu}(x_s^\top \theta) x_s x_s^\top\right) 
  \geq \frac{n}{1+\epsilon} \lambda_{\min}\left(\sum_{x \in \mathcal{X}} \lambda_x \dot{\mu}(x^\top \theta) xx^\top\right)
  \]
\end{itemize}
\end{lemma}

\begin{lemma}
\label{thm:correctness_stopping_condi}
Let $D_\delta(z)$ be any valid confidence width for the estimated model $\hat{\theta}$ on arm $z$ with a probability of at least $1-\delta$, i.e., with a probability of at least $1-\delta$, we have
\begin{align*}
  z^\top\hat{\theta} - D_\delta(z) < z^\top\theta < z^\top\hat{\theta} + D_\delta(z).
\end{align*}
Then the following condition 
\begin{align*}
  D_\delta(z) < \bigl|z^\top\hat{\theta}\bigr|
\end{align*}
is a sufficient condition to guarantee that the reward model $\hat{\theta}$ makes a correct classification on arm $z$ with probability at least $1-\delta$, i.e.,
\begin{align*}
  \sign(z^\top\hat{\theta})=\sign(z^\top\theta).
\end{align*}
\end{lemma}

\begin{proof}
  Suppose that $D_\delta(z) < \bigl|z^\top\hat{\theta}\bigr|$. We now show that $z^\top\hat{\theta}$ and $z^\top\theta$ share the same sign.
  
  \medskip
  
  \noindent\textbf{Case 1:} $z^\top\hat{\theta} > 0$.\\
  Since $D_\delta(z) < z^\top\hat{\theta}$, we have
  \begin{align*}
    z^\top\hat{\theta} - D_\delta(z) > 0.
  \end{align*}
  Hence, with a probability of at least $1-\delta$,
  \begin{align*}
    z^\top\theta > z^\top\hat{\theta} - D_\delta(z) > 0,
  \end{align*}
  so $z^\top\theta > 0$ and the classification is correct.

  \medskip
  
  \noindent\textbf{Case 2:} $z^\top\hat{\theta} < 0$.\\
  Here $-z^\top\hat{\theta} > 0$ and $D_\delta(z) < -z^\top\hat{\theta}$, so
  \begin{align*}
    z^\top\hat{\theta} + D_\delta(z) < 0.
  \end{align*}
  Hence, with a probability of at least $1-\delta$,
  \begin{align*}
    z^\top\theta < z^\top\hat{\theta} + D_\delta(z) < 0,
  \end{align*}
  so $z^\top\theta<0$ and again the classification is correct. In both cases the sign of $z^\top\hat{\theta}$ matches that of $z^\top\theta$.
\end{proof}

\subsection{Proof of Lemma~\ref{lem:greedy_obj_equiv}}
\label{proof:greedy_equiv}
\newtheorem*{app-thm:greedy_equiv}{Lemma~\ref{lem:greedy_obj_equiv}}

\begin{app-thm:greedy_equiv}
  The greedy objective in \cref{alg:greedy_obj} can be equivalently written as
\begin{align*}
  \max_z\min\cbr{-\mathsf{LCB}_t(z), \mathsf{UCB}_t(z)} = \max_z\cbr{D_{\delta,t}(z) - \abr{\hat{\theta}_t^\top z}}.
\end{align*}
\end{app-thm:greedy_equiv}

\begin{proof}

\begin{align*}
  &\max_z\min\cbr{-\mathsf{LCB}_t(z), \mathsf{UCB}_t(z)}\\
  = &\max_z\min\cbr{-(\hat{\theta}_t^\top z - D_{\delta,t}(z)), \hat{\theta}_t^\top z + D_{\delta,t}(z)}\\
  = &\max_z\min\cbr{-\hat{\theta}_t^\top z + D_{\delta,t}(z), \hat{\theta}_t^\top z + D_{\delta,t}(z)}\\
  = &\max_z\cbr{\min\cbr{-\hat{\theta}_t^\top z, \hat{\theta}_t^\top z} + D_{\delta,t}(z)}\\
  = &\max_z\cbr{D_{\delta,t}(z) - \abr{\hat{\theta}_t^\top z}}.
\end{align*}
\end{proof}

\subsection{Proof of Theorem~\ref{thm:greedy_obj_better}}
\label{proof:greedy}
\newtheorem*{app-thm:greedy}{Theorem~\ref{thm:greedy_obj_better}}

\begin{app-thm:greedy}
  For the canonical instance $z_i = e_i$ for $i=1, \ldots, d$ and $\theta^*=(1, \ldots, 1, \eps)^\top$ for $\eps\in(0, 1/4)$, the greedy selection objective $\arg\max_{i \in [d]} \|e_i\|_{H_t'^{-1}}$ has label complexity scaling as $\tilde{O}\rbr{\fr{d}{\eps^2}}$, while with \cref{alg:greedy_obj}, the label complexity scales as $\tilde{O}\rbr{d+\fr{1}{\eps^2}}$.
\end{app-thm:greedy}

\begin{proof}
We begin by establishing a sufficient condition for the algorithms to stop. Our analysis considers the canonical instance where the arm set $\cZ$ consists of the standard basis vectors $\{e_1, \dots, e_d\}$. In this setting, the two algorithms greedily pull arms based on their respective selection rules:
\begin{itemize}
    \item \textbf{G-optimal based method \cite{das2025active}:} At each round $t$, select arm $\arg\max_{i \in [d]} \|e_i\|_{H_t'^{-1}}$, where $H_t' = \sum_{s=1}^{t} \dot{\mu}(z_s^\top \hat{\theta}_t) z_s z_s^\top$.
    \item \textbf{Our method:} At each round $t$, select arm $\arg\max_{i \in [d]} \{D_{\delta,t}(e_i) - |e_i^\top \hat{\theta}_t|\}$.
\end{itemize}

From \cref{thm:correctness_stopping_condi}, the suggested stopping rule is
\begin{align}
    D_{\delta,t}(e_i) - |e_i^\top \hat{\theta}_t| < 0, \quad \text{for all } i \in [d], \label{ieq:stopping_condi_emp}
\end{align}
where $D_{\delta,t}(e_i)$ is the confidence width for arm $e_i$ at round $t$. Since this rule involves the empirical estimate $\hat{\theta}_t$, we first derive a simpler sufficient condition that depends only on the true parameter $\theta^*$.

We show that the condition
\begin{align}
    D_{\delta,t}(e_i) < \frac{1}{2} |e_i^\top \theta^*|, \quad \text{for all } i \in [d], \label{ieq:stopping_condi_non_emp}
\end{align}
is sufficient for the stopping rule in \cref{ieq:stopping_condi_emp} to be satisfied. By definition of the confidence width, we know that with probability at least $1-\delta$, $|e_i^\top \hat{\theta}_t - e_i^\top \theta^*| \leq D_{\delta,t}(e_i)$. Using the triangle inequality, we have:
\begin{align*}
    |e_i^\top \hat{\theta}_t| &\ge |e_i^\top \theta^*| - |e_i^\top \hat{\theta}_t - e_i^\top \theta^*| \\
    &\ge |e_i^\top \theta^*| - D_{\delta,t}(e_i).
\end{align*}
If we assume \cref{ieq:stopping_condi_non_emp} holds, then $|e_i^\top \theta^*| > 2 D_{\delta,t}(e_i)$. Substituting this into the inequality above gives:
\begin{align*}
    |e_i^\top \hat{\theta}_t| > 2 D_{\delta,t}(e_i) - D_{\delta,t}(e_i) = D_{\delta,t}(e_i).
\end{align*}
This directly implies that $D_{\delta,t}(e_i) - |e_i^\top \hat{\theta}_t| < 0$, satisfying the stopping rule.

Now, we use this sufficient condition to determine the number of pulls $t_i$ required for each arm $e_i$. For the canonical instance, \cref{lem:conf_width_canonical} allows us to express the condition in \cref{ieq:stopping_condi_non_emp} as:
\begin{align}
    2.4 \sqrt{\frac{1}{\dot{\mu}(\theta^*_i)t_i}} \sqrt{ \log \frac{2(2 + |\cZ|)}{\delta} } < \frac{1}{2}|\theta^*_i|, \quad \text{for all } i \in [d]. \label{ieq:stopping_each_arm}
\end{align}
Solving for $t_i$, we find that each arm $i$ must be pulled a number of times satisfying:
\begin{align*}
    t_i > \frac{\cC}{\dot{\mu}(\theta^*_i)(\theta^*_i)^2} \log \frac{2(2 + |\cZ|)}{\delta},
\end{align*}
where $\cC$ is a numerical constant.

We can now analyze the total sample complexity for each algorithm.
\paragraph{G-optimal Method}
As established in \cref{lem:g_to_uniform}, the G-optimal method behaves as a round-robin for the canonical instance, meaning $t_i \approx T/d$ for all $i$ after $T$ rounds. To satisfy the stopping condition for all arms, the number of pulls for the arm with the smallest $|\theta_d^*|=\varepsilon$ must meet the requirement. This leads to a total sample complexity $T$ of:
\begin{align*}
    T \approx d \cdot t_d > d \cdot \frac{\cC}{\varepsilon^2} \log\left(\frac{|\cZ|}{\delta}\right) \implies T = \tilde{O}\left(\frac{d}{\varepsilon^2}\right).
\end{align*}
The algorithm continues to pull all arms in a round-robin fashion until the condition is met for every arm, including those that may have satisfied it much earlier.

\paragraph{Our Method}
A fundamental difference in our algorithm is that its selection rule, $\arg\max_i \{D_{\delta,t}(e_i) - |e_i^\top \hat{\theta}_t|\}$, is directly aligned with the stopping rule. Once an arm $e_i$ satisfies the stopping condition $D_{\delta,t}(e_i) - |e_i^\top \hat{\theta}_t| < 0$, its selection metric becomes negative and it will not be pulled again, as the algorithm prioritizes arms with larger, positive values. The algorithm thus focuses its budget on arms that have not yet met the condition.

The total sample complexity is the sum of the required pulls for each arm. Let $|\theta_d^*| = \varepsilon$. The complexity is dominated by the pulls required for this arm, while other arms require fewer pulls.
\begin{align*}
    T = \sum_{i=1}^{d} t_i = \sum_{i=1}^{d} \tilde{O}\left(\frac{1}{(\theta_i^*)^2}\right) = \tilde{O}\left(\sum_{i=1}^{d-1} \frac{1}{(\theta_i^*)^2} + \frac{1}{\varepsilon^2}\right)=\tilde{O}\left(d + \frac{1}{\varepsilon^2}\right).
\end{align*}
This demonstrates a significant improvement over the G-optimal method when $\varepsilon$ is small.
\end{proof}

\begin{lemma}
  \label{lem:conf_width_canonical}
  For the canonical instance where $z_i=e_i$ for $i=1, \cdots, d$, the confidence interval in \cref{thm:conf_logi} can reduce to
  \begin{align*}
  \left| \hat{\theta}_{t,i} - \theta^*_i \right| < 2.4 \sqrt{\fr{1}{\dot{\mu}\rbr{\theta^*_i}t_i}} \sqrt{ \log \frac{2(2 + \abs{\cZ})}{\delta} }.
  \end{align*}

\end{lemma}
\begin{proof}
The result is straightforward by noting the canonical instance and the confidence interval in \cref{thm:conf_logi}
\end{proof}

\begin{lemma}
\label{lem:g_to_uniform}
  For the canonical instance where $z_i=e_i$ for $i=1, \cdots, d$, and $\theta^*=(1, \ldots, 1, \eps)^\top$, the algorithm in \cite{das2025active} pulls arms almost uniformly, i.e.,
  \begin{align*}
  t_1=\cdots=t_{d-1} \quad \text{ and } \quad t_d < t_1 < 2t_d.
\end{align*}
\end{lemma}
\begin{proof}
For the algorithm of \cite{das2025active}, the selection rule is
\begin{align*}
  \arg\max_i \vcn{z_i}_{H'(\cbr{z_s,s\in[t]},\theta^*)^{-1}} = \arg\max_i \fr{1}{\dot{\mu}\rbr{\theta^*_i}t_i}= \arg\min_i \dot{\mu}\rbr{\theta^*_i}t_i.
\end{align*}
Note that we have 
\begin{align*}
  \dot{\mu}\rbr{\theta^*_i}=\dot{\mu}(1) \text{ for } i=1, \cdots, d-1, \text{ and } \dot{\mu}\rbr{\theta^*_d}=\dot{\mu}(\eps).
\end{align*}
Also note that for a small $\eps$, e.g., $\eps\in(0, 1/4)$,
\begin{align*}
  \dot{\mu}(1) < \dot{\mu}(\eps) < 2\dot{\mu}(1).
\end{align*}
This implies that the algorithm will pull the arms almost uniformly. Specifically, if $t_i$ is the number of times arm $i$ is pulled, then we have 
\begin{align*}
  t_1=\cdots=t_{d-1} \quad \text{ and } \quad t_d < t_1 < 2t_d.
\end{align*}
\end{proof}

\section{Lower bound}
\label{app-lowbound}

We provide an information-theoretic lower bound for this problem.

\begin{theorem}
  \label{thm:lower_bound_logi}
  Let $\cZ=\{z_1,\ldots,z_n\}$ and consider the logistic link $\psi(u)=1/(1+e^{-u})$. Suppose $\min_{z\in\cZ}|z^\top\theta^*|>0$, and let $0<\delta\le e^{-1}$. Define the sign-changing alternative set
  \begin{align*}
  \mathcal C_{\theta^*,\cZ}
  &:=\cbr{\theta\in\mathbb{R}^d:\exists\,z\in \cZ,\;\sign(z^\top\theta^*)
  \neq\sign(\theta^\top z)}.
  \end{align*}
  For any strategy that is uniformly $\delta$-PAC over logistic parameters with no zero arm margins, the expected label complexity satisfies
  \begin{align*}
  \mathbb E_{\theta^*}[\tau]
  \ge\;
  \log\frac1{2.4\,\delta}\inf_{\lambda\in\Delta(\cZ)}\;\sup_{\theta\in\mathcal C_{\theta^*,\cZ}}
    \fr{1}{\sum_{i=1}^n \lambda_i\,
    D_{\mathrm{KL}}\rbr{P_{\theta^*,i}\,\|\,P_{\theta,i}}},
  \end{align*}
  where $P_{\theta,i}=\mathrm{Bern}\rbr{\psi(z_i^\top\theta)}$. Alternative parameters are unrestricted apart from the sign change; zero-margin boundary points are interpreted by continuity.
\end{theorem}

\begin{proof}
Let $T_i$ be the number of queries to arm $z_i$, so $\tau=\sum_{i=1}^n T_i$. If $\mathbb E_{\theta^*}[\tau]=\infty$, the result is immediate. Otherwise, the transportation lemma of \citet{kaufmann2016complexity}, applied to the event that the returned classifier has the same arm signs as $\theta^*$, gives
\begin{align*}
\sum_{i=1}^n \mathbb E_{\theta^*}[T_i]\,
D_{\mathrm{KL}}\rbr{P_{\theta^*,i}\,\|\,P_{\theta,i}}
\;\ge\;\log\frac1{2.4\,\delta}
\end{align*}
for every sign-changing alternative $\theta$ with no zero arm margins. Since $\cZ$ is finite and the true margins are nonzero, the boundary points in $\mathcal C_{\theta^*,\cZ}$ can be approached by such alternatives. Continuity of the Bernoulli KL divergence therefore extends the inequality to the whole set.

The preceding inequality shows that $(\mathbb E_{\theta^*}[T_1],\ldots,\mathbb E_{\theta^*}[T_n])$ is a feasible solution to the following optimization problem,
\begin{align*}
    &\min_{\mathbf t\in[0,\infty)^n}\;\sum_{i=1}^n t_i
    \\
    \text{s.t.}\quad
    &\inf_{\theta\in\mathcal C_{\theta^*,\cZ}}\sum_{i=1}^n t_i\,
    D_{\mathrm{KL}}\rbr{P_{\theta^*,i}\,\|\,P_{\theta,i}}
    \;\ge\;
    \log\frac{1}{2.4\,\delta}.
\end{align*}
Let $\mathbf t^*$ be an optimal solution, which exists since the feasible set is closed and has a nonempty compact objective sublevel set. Since the expected-count vector is feasible, we have
\begin{align*}
  \sum_{i=1}^n\mathbb E_{\theta^*}[T_i]
  \;\ge\;
  \sum_{i=1}^n t_i^*.
\end{align*}
Normalizing $\mathbf t^*$ by its positive total gives
\begin{align*}
  \inf_{\theta\in\mathcal C_{\theta^*,\cZ}}
  \sum_{i=1}^n
    \frac{t_i^*}{\sum_{j=1}^n t_j^*}
    D_{\mathrm{KL}}\rbr{P_{\theta^*,i}\,\|\,P_{\theta,i}}
  \;\ge\;
  \frac{\log\rbr{1/(2.4\delta)}}{\sum_{j=1}^n t_j^*}
  \;\ge\;
  \frac{\log\rbr{1/(2.4\delta)}}{\sum_{j=1}^n\mathbb E_{\theta^*}[T_j]}.
\end{align*}
Since $\mathbf t^*/\sum_{j=1}^n t_j^*$ belongs to $\Delta(\cZ)$, we conclude
\begin{align*}
  \sup_{\lambda\in\Delta(\cZ)}\;\inf_{\theta\in\mathcal C_{\theta^*,\cZ}}
    \sum_{i=1}^n\lambda_i\,
    D_{\mathrm{KL}}\rbr{P_{\theta^*,i}\,\|\,P_{\theta,i}}
  \;\ge\;
  \frac{\log\rbr{1/(2.4\delta)}}{\sum_{i=1}^n\mathbb E_{\theta^*}[T_i]}.
\end{align*}
Taking reciprocals and rearranging the terms, we obtain
\begin{align*}
  \mathbb E_{\theta^*}[\tau]=\sum_{i=1}^n\mathbb E_{\theta^*}[T_i]
  \ge\;
  \log\frac1{2.4\,\delta}\inf_{\lambda\in\Delta(\cZ)}\;\sup_{\theta\in\mathcal C_{\theta^*,\cZ}}
    \fr{1}{\sum_{i=1}^n\lambda_i\,
    D_{\mathrm{KL}}\rbr{P_{\theta^*,i}\,\|\,P_{\theta,i}}}.
\end{align*}
\end{proof}

The instance-dependent quantity \(\rho^*\) is analogous to complexity terms that arise in some best-arm identification problems with linear structures \citep{jun2021improved, fiez2019sequential, NEURIPS2024_zhao}. This term captures the problem's complexity through the geometry of the arm set and the true parameter \(\theta^*\). For linear bandits, \citet{fiez2019sequential, NEURIPS2024_zhao} have shown that such a term characterizes the optimal sample complexity via a minimax lower bound. Here we make the comparison precise for the logistic link. Its Fisher information depends on the true logits, so the comparison with the KL-based lower bound must account for this curvature. Under the bounded-logit conditions in \cref{thm:bounded_logit_optimality}, the following proof gives this comparison with an explicit dependence on the logit upper bound.

\subsection{Proof of Theorem~\ref{thm:bounded_logit_optimality}}
\label{proof:bounded_logit_optimality}
\newtheorem*{app-thm:bounded_logit_optimality}{Theorem~\ref{thm:bounded_logit_optimality}}

\begin{app-thm:bounded_logit_optimality}[Leading-term comparison for bounded logits]
Consider the logistic link $\psi(u)=1/(1+e^{-u})$. Let $\cZ=\{z_1,\ldots,z_n\}$ span $\mathbb R^d$, and suppose
\begin{align*}
  \max_{z\in\cZ}|z^\top\theta^*|\le B,
  \qquad \kappa_B:=\psi(B)(1-\psi(B)).
\end{align*}
Then
\begin{align*}
  T_{\mathrm{KL}}^\star\ge8\kappa_B\rho^*.
\end{align*}
\end{app-thm:bounded_logit_optimality}

\begin{proof}
Fix sampling weights $\lambda$, and write $H_\lambda:=H(\lambda,\theta^*)$. First suppose that $H_\lambda$ has an inverse, and define
\begin{align*}
  V_\lambda:=\sum_{i=1}^n\lambda_i z_i z_i^\top,
  \qquad
  R_\lambda:=\max_{z\in\cZ}
  \frac{z^\top H_\lambda^{-1}z}{|z^\top\theta^*|^2}.
\end{align*}
When $H_\lambda$ has no inverse, we set $R_\lambda=+\infty$ and use the same convention in the objective defining $\rho^*$. Thus $\rho^*=\inf_\lambda R_\lambda$.

Choose an arm $z_\lambda$ with the largest ratio in the definition of $R_\lambda$. Let $g=z_\lambda^\top\theta^*$ be its true logit and write $q=z_\lambda^\top H_\lambda^{-1}z_\lambda$, so $R_\lambda=q/g^2$. For $\eta>0$, we change the parameter by $v_\eta$ as follows:
\begin{align*}
  v_\eta:=-(1+\eta)\frac{g}{q}H_\lambda^{-1}z_\lambda,
  \qquad \theta_\eta:=\theta^*+v_\eta.
\end{align*}
This gives $z_\lambda^\top\theta_\eta=-\eta g$, so the preference for $z_\lambda$ is reversed and $\theta_\eta\in\mathcal C_{\theta^*,\cZ}$. Each new logit is its nonzero true logit plus a constant times $1+\eta$, so it can be zero for at most one value of $\eta$. There are finitely many arms, so we can choose positive $\eta$ arbitrarily close to zero without any ties. The changed logits do not need to satisfy the bound $B$.

To bound the KL divergence, define $A(u)=\log(1+e^u)$. Its derivatives satisfy $A'(u)=\psi(u)$ and $A''(u)=\psi'(u)\le1/4$. By Taylor's theorem, for any two logits $a,b\in\mathbb R$,
\begin{align*}
  D_{\mathrm{KL}}\rbr{\mathrm{Bern}(\psi(a))\,\|\,\mathrm{Bern}(\psi(b))}
  =A(b)-A(a)-A'(a)(b-a)
  \le\frac{(a-b)^2}{8}.
\end{align*}
The bound on the true logits gives $\psi'(z_i^\top\theta^*)\ge\kappa_B$ for every arm. Hence $H_\lambda\succeq\kappa_B V_\lambda$, so $v_\eta^\top V_\lambda v_\eta\le v_\eta^\top H_\lambda v_\eta/\kappa_B$. The logit of arm $z_i$ changes by $z_i^\top v_\eta$. Applying the KL bound above to each arm and multiplying by its sampling weight $\lambda_i$, we obtain
\begin{align*}
  I_\lambda(\theta_\eta)
  &:=\sum_{i=1}^n\lambda_i
  D_{\mathrm{KL}}\rbr{P_{\theta^*,i}\,\|\,P_{\theta_\eta,i}}\\
  &\le\frac18 v_\eta^\top V_\lambda v_\eta
  \le\frac{1}{8\kappa_B}v_\eta^\top H_\lambda v_\eta
  =\frac{(1+\eta)^2}{8\kappa_B}\frac{g^2}{q}
  =\frac{(1+\eta)^2}{8\kappa_B R_\lambda}.
\end{align*}
For the last two equalities, we substitute $v_\eta$ and use $R_\lambda=q/g^2$. Since $\theta_\eta\in\mathcal C_{\theta^*,\cZ}$, the supremum below is at least $1/I_\lambda(\theta_\eta)$. Taking the reciprocal of the bound above and letting $\eta$ approach zero gives
\begin{align*}
  \sup_{\theta\in\mathcal C_{\theta^*,\cZ}}\frac{1}{I_\lambda(\theta)}
  \ge8\kappa_B R_\lambda.
\end{align*}

Finally, suppose $H_\lambda$ has no inverse. This means that some changes to the parameter leave the logits of all arms with positive sampling weight unchanged. More precisely, there is a nonzero vector $u$ such that $z_i^\top u=0$ whenever $\lambda_i>0$. For any $t$, replacing $\theta^*$ by $\theta^*+tu$ therefore leaves the label probabilities of all these arms unchanged.

However, the full arm set $\cZ$ spans $\mathbb R^d$, so some arm must have $z_i^\top u\ne0$. Its new logit is $z_i^\top\theta^*+t z_i^\top u$, and we can choose $t$ to reverse its sign. We can also avoid ties by excluding the finitely many values of $t$ that make any arm's logit zero. The resulting parameter changes a preference but gives exactly the same label probabilities on every arm with positive sampling weight. Its total KL divergence under $\lambda$ is therefore zero. With $1/0=+\infty$, the left-hand side above is infinite, as is $R_\lambda$ by our convention.

We have proved the comparison for every choice of sampling weights. Taking the infimum over $\lambda$ gives
\begin{align*}
  T_{\mathrm{KL}}^\star\ge8\kappa_B\inf_\lambda R_\lambda
  =8\kappa_B\rho^*.
\end{align*}
\end{proof}

\section{Experiments}
\label{app-experiment}

We evaluate both of our proposed algorithms against five baselines using real-world and synthetic preference data. We evaluate the greedy algorithm in \cref{alg:greedy_obj} on the Anthropic helpful and harmless (Anthropic HH) dataset \citep{bai2022training}, using classification accuracy on held-out preference pairs and best-of-$N$ response selection with two reward-model judges. We also evaluate the experimental design algorithm in \cref{alg:exp_design_logi} on synthetic instances with known true preferences. Specifically, we consider the following five algorithms, including existing works on preference learning and established active learning methods adapted to this setting, for comparison with our proposed algorithms:
\begin{itemize}
  \item \textbf{Random Sampling:} Arms are selected uniformly at random at each round.

  \item \textbf{Uncertainty Sampling:} It selects arms closest to the decision boundary based on the current estimate, i.e., those with the smallest absolute estimated reward: $\arg\min_z \abr{z^\top\hat{\theta}_t}$. A simple and widely adopted heuristic in active learning \citep{lewis1995sequential, tong2001support, mussmann2018uncertainty, raj2022convergence}. It has proven effective in various settings.
  
  \item \textbf{Selective Sampling:} A popular active learning method for the streaming setting where the learner decides whether to query a human labeler as each arm arrives. We adapt it to our setting by iterating through the arms in a shuffled order and selecting the first arm $z$ whose lower confidence bound is below a threshold, e.g., $\mathsf{LCB}_t(z) < 0.1$. Note that all true rewards $z^\top\theta^*$ are positive in our setup, since we construct each $z$ as the difference between the chosen and rejected responses.
  
  \item \textbf{APO \citep{das2025active}:} It pulls the arm that is most uncertain under the current model, measured by the width of the confidence ellipsoid in its direction: $\arg\max_z \vcn{z}_{H'(\cbr{z_s,s\in[t]},\hat{\theta}_t)^{-1}}$. The same approach is also used in \citet{li2025provably}.
  
  \item \textbf{D-optimal design:} This common experimental design criterion, used in \citet{liu2024dual,kveton2025active}, selects the arm that greedily maximizes the determinant of the information matrix, i.e., $\arg\max_z \det\left(\sum_{s=1}^t \dot{\mu}(z_s^\top \hat{\theta}_t) z_s z_s^\top\right)$. Note that \citet{kveton2025active} propose an equivalent formulation that avoids computing the determinant directly.
\end{itemize}
For the Anthropic HH evaluations, we use Gemma 2 (\texttt{gemma-2-2b-it}) \citep{gemma_2024} to generate text embeddings. We clarify that our setup does not involve full parameter fine-tuning, rather, it is equivalent to using a frozen backbone while the learning process focuses on the reward head, as suggested by \citet{kumar2022fine}. The comparison focuses on the performance of different preference data selection methods. Freezing the LLM backbone and isolating the linear reward head enables us to rigorously bound the sample complexity of our active learning algorithm, ensuring that these theoretical bounds are not confounded by the unstable feature distortion often associated with underlying network updates.

\subsection{Preference Classification Accuracy}
\label{app-exp-accuracy}

Following practical considerations for real-world applications, we implement all algorithms in a batched setting, as a fully sequential, round-by-round implementation is often infeasible due to labeling latency \citep{scheid2024optimal}. In each epoch, we sample a batch of 50 arms to query, collect their labels, and then update the model. All reported results are averaged over 50 independent runs. Note that any given dataset represents a single realization of the underlying label noise, meaning the chosen and rejected labels for each pair are fixed. The randomness in our experimental results, therefore, comes from the random train-test split performed at the start of each independent run. 

We report the classification accuracy of the trained reward models on the Anthropic HH test set in \cref{fig:exp_results}. Our proposed method consistently outperforms the baselines. One interesting observation from our experiments is that methods like D-optimal design perform well early on, but then struggle to keep improving near the end. This is expected because D-optimal design is non-adaptive here and does not use any of the label information it collects. In contrast, Uncertainty Sampling is adaptive, so it eventually reaches performance comparable to ours, but it tends to lag at the beginning. Our method achieves the best of both behaviors. In fact, this is reflected in \cref{eq:greedy_eq}: one term matches the signal used by Uncertainty Sampling, while the other mirrors the criterion behind D-optimal design.

\begin{figure}[H]
    \centering
    \includegraphics[width=0.62\linewidth]{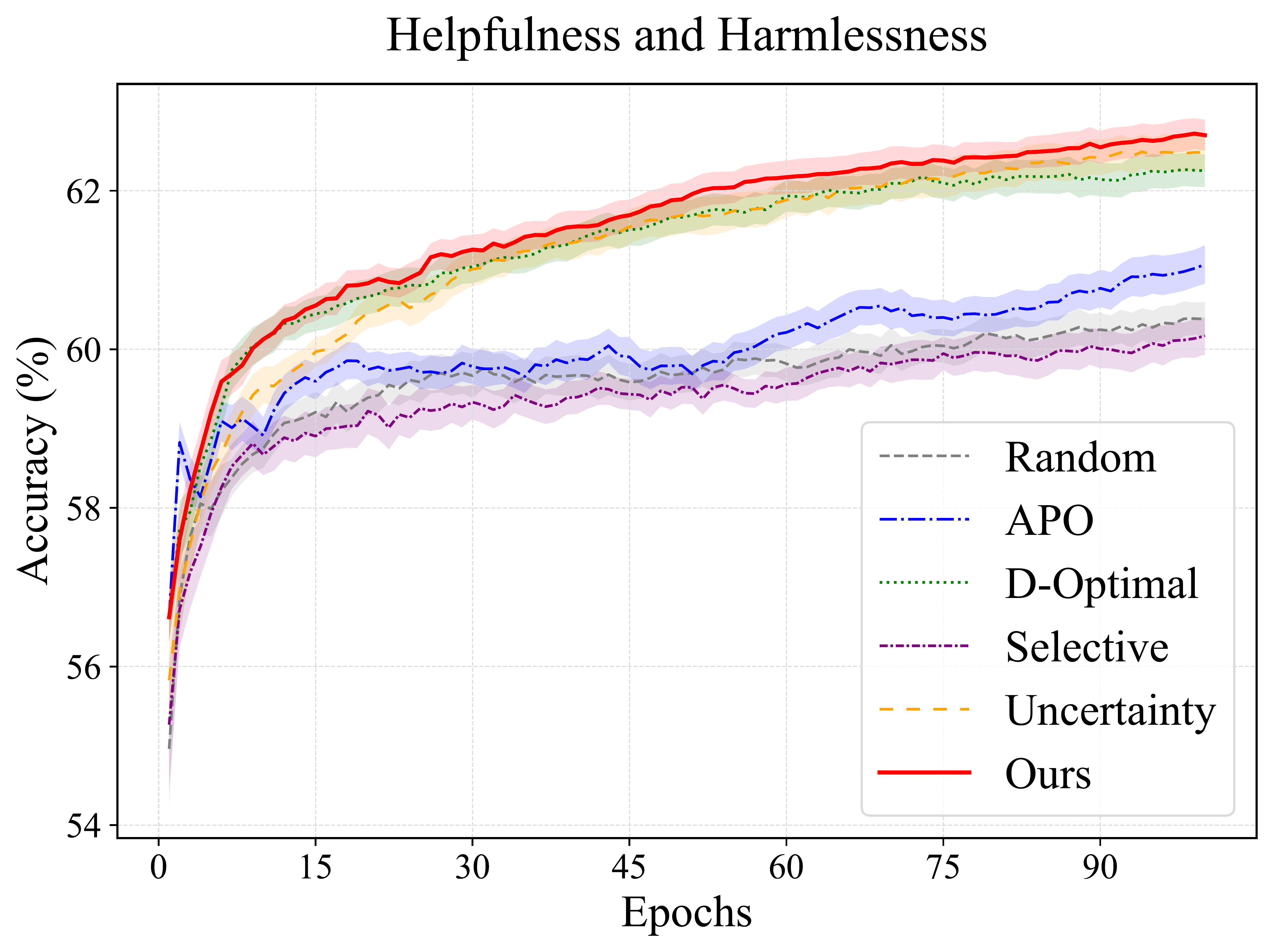}
    \caption{Preference classification accuracy on Anthropic HH, averaged over $50$ independent runs.}
    \label{fig:exp_results}
    \label{fig:exp_hh}
\end{figure}

\subsection{Best-of-\texorpdfstring{$N$}{N} Response Selection}
\label{app-exp-best-of-n}

We also evaluate whether the trained reward models select better responses from a set of candidates for the same prompt. We reserve $1000$ held-out prompts that are disjoint from the selection pool of preference pairs. For each of $50$ random seeds, we redraw an $80/20$ split of this pool and train the reward head using batches of $50$ queries.

For each held-out prompt $x$, we independently sample $N=8$ responses from \texttt{Qwen2.5-7B-Instruct}, with temperature $0.8$ and top-$p=0.95$. Each candidate pool is generated once and reused across all methods, seeds, and labeling budgets. Let $h(x,a)$ denote the embedding produced by the frozen Gemma 2 backbone for prompt $x$ and response $a$, so the trained reward head assigns score $\hat{\theta}^\top h(x,a)$ to each response. It selects
\begin{align*}
  \hat a(x)=\arg\max_{a\in\{a_1,\ldots,a_N\}}\hat{\theta}^\top h(x,a).
\end{align*}
We score the same responses with two open source reward models as judges: \texttt{Skywork-Reward-V2-Llama-3.1-8B} (Skywork 8B) and \texttt{INF-ORM-Llama3.1-70B} (INF-ORM 70B, int8 quantization was applied to the INF-ORM 70B judge). For each judge, scores are computed once for each candidate and reused in every comparison.

We compute the relative improvement for each judge. Let $\bar r_{\mathrm{ours}}$ and $\bar r_{\mathrm{base}}$ be the mean judge scores of the responses selected by our method and a baseline at the same labeling budget, averaged over the $1000$ held-out prompts and $50$ runs. Let $\bar r_{\mathrm{rand}}$ be the corresponding mean score when a response is chosen uniformly at random from each candidate pool; for each prompt, this is the average judge score of its $N$ candidates. We report
\begin{align*}
  \text{Relative improvement (\%)}
  =100\,\frac{\bar r_{\mathrm{ours}}-\bar r_{\mathrm{base}}}
  {\bar r_{\mathrm{base}}-\bar r_{\mathrm{rand}}}.
\end{align*}
This measures the improvement over a baseline relative to what that baseline gains from using a reward model instead of choosing a response at random. Random response selection is distinct from the Random Sampling baseline, which trains a reward model on randomly queried preference pairs.

The relative improvement of our method over each baseline is reported in \cref{fig:bon_paired}. For both judges, the results favor our method against all five baselines across the three labeling budgets. The gains are largest over Selective Sampling and smaller over D-optimal design. Some confidence intervals overlap zero, but the overall trend is consistent across the two judges.

\begin{figure}[H]
  \centering
  \begin{subfigure}{\linewidth}
    \centering
    \includegraphics[width=\linewidth]{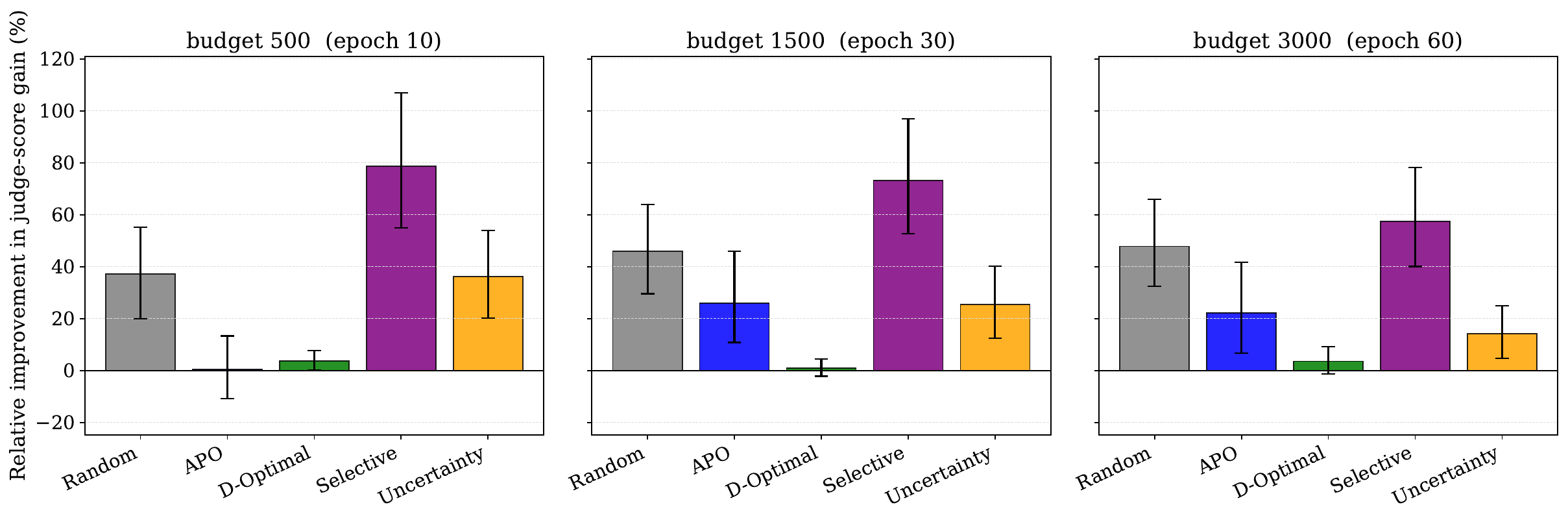}
    \caption{Skywork 8B judge.}
    \label{fig:bon_skywork}
  \end{subfigure}
  \par\medskip
  \begin{subfigure}{\linewidth}
    \centering
    \includegraphics[width=\linewidth]{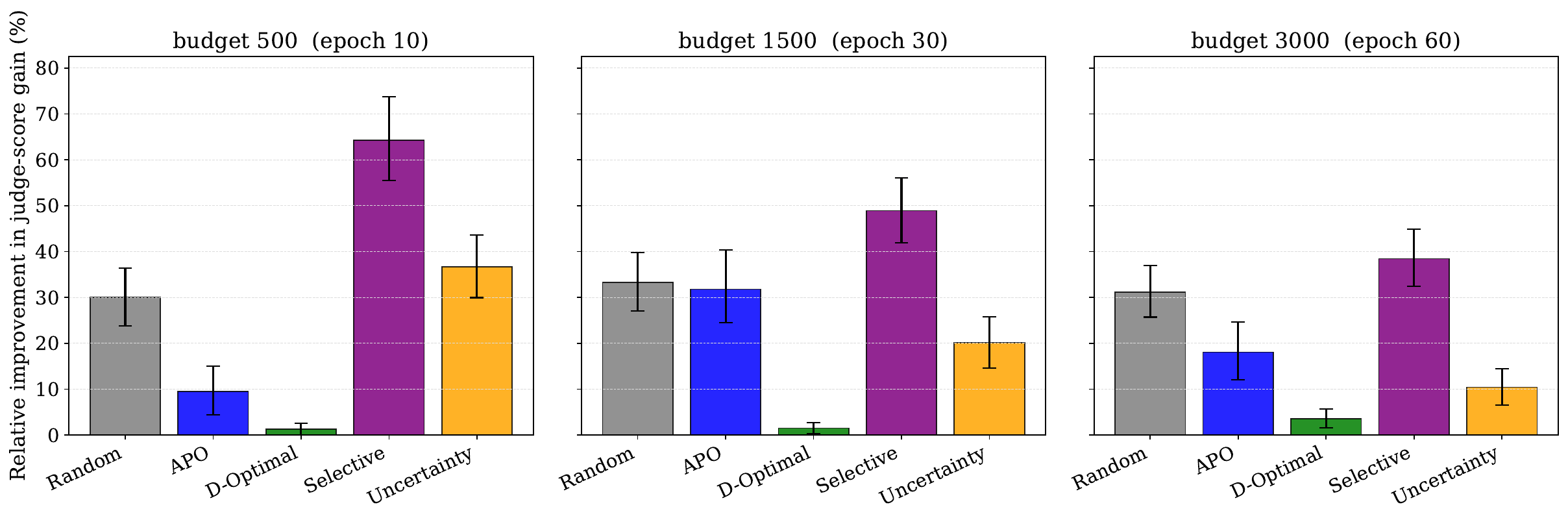}
    \caption{INF-ORM 70B judge.}
    \label{fig:bon_inform70b}
  \end{subfigure}
  \caption{Best-of-$N$ response selection on Anthropic HH with $N=8$, evaluated by two reward-model judges. Each bar shows the percentage improvement of our method over the indicated baseline in judge-score gain above a uniformly random response from the same candidate pool. Scores are averaged over $1000$ held-out prompts and $50$ runs. Positive values favor our method.}
  \label{fig:bon_paired}
\end{figure}

\subsection{Synthetic Preference Data}
\label{app-exp-synthetic}

We also evaluate the experimental design algorithm in \cref{alg:exp_design_logi} against the same five baselines on synthetic preference data. 

We produce a pool of $n=100$ arms in $d=15$ dimensions. In each run, we sample a random unit vector $\theta^*$ and construct the arms as
\begin{align*}
  z=m\theta^*+\sqrt{1-m^2}\,v,
  \qquad \|v\|=1,\quad v^\top\theta^*=0,
\end{align*}
so that $z^\top\theta^*=m$, where $m$ may be positive or negative. The pool contains two groups:
\begin{itemize}
  \item \textbf{60 easy arms:} Unit norm and large absolute margins ($0.61$--$1.00$), with random positive or negative signs. Each arm has its own randomly sampled direction $v$ orthogonal to $\theta^*$.
  \item \textbf{40 hard arms:} Random positive or negative signs, with all arms sharing one orthogonal direction up to sign, $v$ or $-v$. We use pre-scaling absolute margins $0.30$--$0.60$ and multiply each entire vector by $0.15$, giving each arm a norm of $0.15$ and absolute margins of approximately $0.045$--$0.09$.
\end{itemize}
We use a practical version of \cref{alg:exp_design_logi}. It uses $5d$ warm-up labels. The precision starts at $\eps_1=2\max_z|z^\top\hat{\theta}_0|$ and is multiplied by $0.7$ when the round's target is met. The target for each arm is $\max\{|z^\top\hat{\theta}|,\eps_\ell\}$. We fit the regularized logistic MLE using all labels collected so far, and account for them when computing the next design and round size. 

At each budget $t$, we report the percentage of arms satisfying $\widetilde D_t(z)<|z^\top\hat{\theta}_t|$, where
\begin{align*}
  \widehat H_t
  &=\sum_{s=1}^t\psi'(z_s^\top\hat{\theta}_t)z_s z_s^\top+\alpha I_d,
  &\widetilde D_t(z)
  &=\sqrt{\log\frac{2n(2+t)}{\delta}}\;\|z\|_{\widehat H_t^{-1}}.
\end{align*}
Here $\hat{\theta}_t$ is the regularized MLE on all collected labels. The criterion checks whether an arm's estimated sign is separated from zero by more than its estimated uncertainty. The plot refers to these arms as certified.

In \cref{fig:synthetic_results}, our method outperforms all five baselines under the empirical criterion at both $10048$ and $20000$ labels. Uncertainty Sampling concentrates on the small-margin hard arms, which share an orthogonal direction. Other directions therefore receive few queries, leaving the estimated confidence intervals for many easy arms wide. Consequently, its performance is significantly worse than that of the other methods.

\begin{figure}[H]
  \centering
  \includegraphics[width=\linewidth]{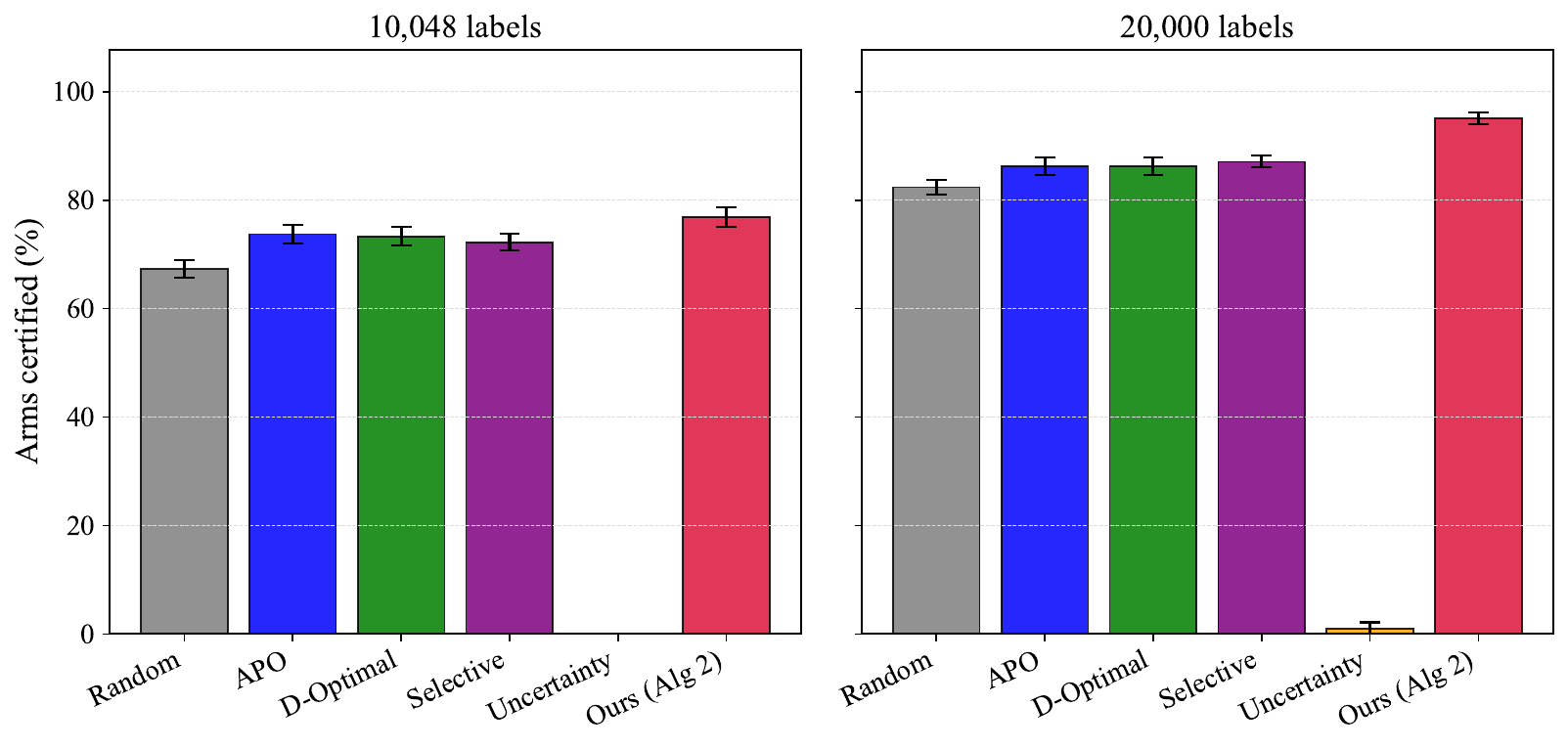}
  \caption{Synthetic preference learning with $n=100$ arms and $d=15$. Bars show the percentage of arms satisfying the empirical confidence-width criterion, averaged over $50$ runs.}
  \label{fig:synthetic_results}
\end{figure}

\end{document}